%% file: main.tex
\documentclass[11pt]{article} %
\usepackage[T1]{fontenc}
\usepackage{lmodern}
\usepackage{amsmath,amssymb,amsthm,bm}
\usepackage{natbib}
\usepackage[margin=1in]{geometry}
\usepackage{subcaption}

\usepackage{hyperref,array}
\usepackage{url}
\usepackage{xcolor}
\usepackage{booktabs}
\usepackage{graphicx}
\usepackage{enumitem}
\usepackage{xargs}
\usepackage{caption}
\usepackage{booktabs}

\hypersetup{colorlinks=true,linkcolor=black,citecolor=black,urlcolor=blue}

\newcommand{\luca}[1]{}
\newcommand{\tyler}[1]{}
\newcommand{\umut}[1]{}
\newcommand{\alain}[1]{}

\newtheorem{theorem}{Theorem}
\newtheorem{lemma}{Lemma}

\newtheorem{proposition}{Proposition}

\newtheorem{corollary}{Corollary}
\newtheorem{definition}{Definition}

\newcommand{\supp}{\operatorname{supp}}
\newcommand{\Unif}{\mathrm{Unif}}
\newcommand{\data}{\mu_{\mathrm{data}}}
\newcommand{\kll}{\mathrm{KL}}
\newcommand{\keff}{k_{\mathrm{eff}}}
\newcommand{\KL}{\mathrm{KL}}

\newcommand{\cX}{\mathcal X}

\newcommand{\cP}{\mathcal P}
\newcommand{\cE}{\mathcal E}
\newcommand{\mudata}{\mu_{\mathrm{data}}}
\newcommand{\muhat}{\hat\mu^{(\gamma)}_0}
\newcommand{\E}{\mathbb E}
\newcommand{\Prob}{\mathbb P}
\newcommand{\koeff}{k^{o}_{\mathrm{eff}}}

\let\cedilla\c

\input{def}

\title{Discrete Diffusion Models are\\ Minimax Optimal Under Sparsity}
\author{
\normalsize Luca Blunda$^{1,}$\footnotemark[1] \quad Tyler Farghly$^1$ \quad Umut \cedilla{S}im\cedilla{s}ekli$^{1,}$\footnotemark[2] \quad Alain Durmus$^{2,}$\footnotemark[2]\\
\normalsize $^{1}$INRIA - École Normale Supérieure - PSL Research University -  CNRS, France\\
\normalsize $^{2}$École Polytechnique - CMAP - IP Paris, France
}
\date{}

\begin{document}

\maketitle

\footnotetext[1]{Correspondence to \texttt{luca.blunda@inria.fr}.}
\footnotetext[2]{Authors contributed equally to this work.}

\begin{abstract}
    Discrete diffusion models have emerged as a leading alternative to autoregressive models for generating text and other categorical data, yet their theoretical foundations remain less developed than those of their continuous counterparts. For continuous diffusion models, a prominent line of theory suggests that these models can adapt to low-dimensional structure in the data, thereby improving statistical efficiency and generalization. In this work, we investigate the discrete analogue of this phenomenon by studying the statistical rates of discrete diffusion models under sparsity assumptions. By analyzing a nonparametric estimator derived from a regularized score-entropy objective, we establish minimax-optimal rates in the sparse regime. The proposed regularization can be viewed as an adaptation of label smoothing techniques introduced for classification, and does not require the knowledge of the sparsity parameter. This shows that discrete diffusion models can adapt intrinsically to sparsity.
\end{abstract}

\input{sections/introduction}
\input{sections/preliminaries}
\input{sections/statistical_rate_first_part}

\input{sections/statistical_rates}

\input{sections/experiments}

\section{Conclusion}
In this work, we analyze the statistical efficiency of discrete diffusion models under sparsity. For the masked setting, we develop a regularized nonparametric estimator of the score function derived by treating it as a classifier and applying label smoothing. By regularizing the score in this fashion, we establish upper bounds on the KL divergence which achieve minimax optimal rates in the class of sparse distributions. We show that the same rates hold for uniform diffusion: the regularized score is again the exact score of \eqref{eq:intro_smoothed_estimator}, but the smoothing it induces has no interpretation as label smoothing of a classifier. Since the minimax lower bound is the same for both kernels, our results rule out a separation between masked and uniform diffusion at the level of the minimax risk over sparse distributions. Whether the gap observed in practice comes from approximation or optimization of the score, is a question we leave open. While our primary focus is providing theoretical justification for the success discrete diffusion, our work also suggests two practical methods of regularization for masked diffusions---label smoothing of the score and regularization via the uniform score-entropy loss---whose rigorous empirical evaluation we leave for future work.

\subsubsection*{Acknowledgments}
U.S. is partially supported by the French government under the management of Agence Nationale de la Recherche as
part of the “Investissements d’avenir” program, reference ANR-19-P3IA-0001 (PRAIRIE 3IA Institute). L.B., T.F.,
and U.S. are supported by the European Research Council Starting Grant DYNASTY – 101039676. A.D. is supported
by the France 2030 program with the reference ANR-25-PEIA-0001 (THEOREM project). A.D. is funded by the
European Union (ERC-2022-SYG-OCEAN-101071601). Views and opinions expressed are however those of the
author(s) only and do not necessarily reflect those of the European Union or the European Research Council Executive
Agency. Neither the European Union nor the granting authority can be held responsible for them. A.D. is supported by Paris and Agence Nationale de la Recherche (Grant 11-LABX-0047). This work received government funding
administered by the National Research Agency (ANR) under the France 2030 program ”Hi! PARIS”, grant number
ANR-23-IACL-0005.

\bibliography{biblio}
\bibliographystyle{iclr2027_conference}

\newpage
\appendix
\input{sections/appendix}
\input{sections/appendix_experiments}

\end{document}

%% file: def.tex
\newcommand{\dd}{\mathrm{d}}

\newcommandx{\functionspace}[2][1=+]{\mathbb{F}_{#1}(#2)}

\newcommandx{\VarDeux}[3][3=]{\operatorname{Var}^{#3}_{#1}\left\{#2 \right\}}

\newcommand{\LeftEqNo}{\let\veqno\@@leqno}

\newcommand{\N}{\ensuremath{\mathbb{N}}}

\newcommandx{\Vnorm}[2][1=V]{\| #2 \|_{#1}}
\newcommandx{\VnormEq}[2][1=V]{\left\| #2 \right\|_{#1}}
\newcommandx{\norm}[2][1=]{\ifthenelse{\equal{#1}{}}{\left\Vert #2 \right\Vert}{\left\Vert #2 \right\Vert^{#1}}}
\newcommandx{\normLigne}[2][1=]{\ifthenelse{\equal{#1}{}}{\Vert #2 \Vert}{\Vert #2\Vert^{#1}}}

\newcommandx\probaMarkovTilde[2][2=]
{\ifthenelse{\equal{#2}{}}{{\widetilde{\mathbb{P}}_{#1}}}{\widetilde{\mathbb{P}}_{#1}\left[ #2\right]}}

\def\eqsp{\;}

\newcommandx{\weight}[2][2=n]{\omega_{#1,#2}^N}

\newcommandx\sequence[3][2=,3=]
{\ifthenelse{\equal{#3}{}}{\ensuremath{\{ #1_{#2}\}}}{\ensuremath{\{ #1_{#2}, \eqsp #2 \in #3 \}}}}
\newcommandx\process[3][2=,3=]
{\ifthenelse{\equal{#3}{}}{\ensuremath{( #1_{#2})}}{\ensuremath{( #1_{#2})_{ #2 \in #3}}}}
\newcommandx\processt[2][2=]
{\ifthenelse{\equal{#2}{}}{\ensuremath{( #1_{t})}}{\ensuremath{( #1_{t})_{ t \in #2}}}}
\newcommandx\sequenceD[3][2=,3=]
{\ifthenelse{\equal{#3}{}}{\ensuremath{\{ #1_{#2}\}}}{\ensuremath{( #1)_{ #2 \in #3} }}}

\newcommandx{\sequencen}[2][2=n\in\N]{\ensuremath{\{ #1_n, \eqsp #2 \}}}
\newcommandx\sequenceDouble[4][3=,4=]
{\ifthenelse{\equal{#3}{}}{\ensuremath{\{ (#1_{#3},#2_{#3}) \}}}{\ensuremath{\{  (#1_{#3},#2_{#3}), \eqsp #3 \in #4 \}}}}
\newcommandx{\sequencenDouble}[3][3=n\in\N]{\ensuremath{\{ (#1_{n},#2_{n}), \eqsp #3 \}}}

\newcommand{\opnorm}[1]{{\left\vert\kern-0.25ex\left\vert\kern-0.25ex\left\vert #1
    \right\vert\kern-0.25ex\right\vert\kern-0.25ex\right\vert}}

\newcommandx{\CPE}[3][1=]{{\mathbb E}_{#1}\left[\left. #2 \, \middle \vert \, #3 \right. \right]} %

\newcommandx{\CPELigne}[3][1=]{{\mathbb E}_{#1}[\left. #2 \,  \vert \, #3 \right. ]} %

\newcommandx{\CPVar}[3][1=]{\mathrm{Var}^{#3}_{#1}\left\{ #2 \right\}}
\newcommand{\CPP}[3][]
{\ifthenelse{\equal{#1}{}}{{\mathbb P}\left(\left. #2 \, \right| #3 \right)}{{\mathbb P}_{#1}\left(\left. #2 \, \right | #3 \right)}}

\newcommandx{\osc}[2][1=]{\mathrm{osc}_{#1}(#2)}

\def\c{c}

\newcommand\coupling[2]{\Gamma(\mu,\nu)}
\def\supp{\mathrm{supp}}

\renewcommand{\geq}{\geqslant}
\renewcommand{\leq}{\leqslant}

\newcommandx{\wasserstein}[3][1=\distance,3=]{\mathscr{W}_{#1}^{#3}\left(#2\right)}
\newcommandx{\wassersteinLigne}[3][1=\distance,3=]{\mathscr{W}_{#1}^{#3}(#2)}
\newcommandx{\wassersteinD}[1][1=\distance]{\mathscr{W}_{#1}}
\newcommandx{\wassersteinDLigne}[1][1=\distance]{\mathscr{W}_{#1}}

\def\pre_Stif{\mathcal{V}}

\def\bfo{\mathbf{o}}
\def\Vlyap{\mathfrak{V}}

\newcommandx{\Voi}[1][1=i]{\Vlyap_{\bfo,#1}}
\newcommandx{\Vlyapc}[2][1=\bfo,2=i]{\Vlyap_{#1,#2}}
\newcommandx{\tVlyapc}[2][1=\bfo,2=i]{\tilde{\Vlyap}_{#1,#2}}

\def\Wlyap{\mathfrak{W}}

\newcommandx{\Woi}[1][1=i]{\Wlyap_{\bfo,#1}}
\newcommandx{\Wlyapc}[2][1=\bfo,2=i]{\Wlyap_{#1,#2}}

\def\KL{\mathrm{KL}}

%% file: sections/introduction.tex
\section{Introduction}

Discrete diffusion models learn distributions on categorical data by reversing a Markov noising
process \citep{austin2021structured,lou2024discrete}. In language modeling, they provide an alternative
to autoregressive sequence generation \citep{sahoo2024simple}. We study their finite-sample estimation
error in Kullback--Leibler (KL) divergence under two standard noising mechanisms: \emph{uniform
corruption}, which replaces tokens with uniformly sampled tokens, and \emph{masked corruption}, which
replaces them with an absorbing mask token. Our objective is to determine whether these mechanisms
attain the optimal estimation rates when the data distribution has an unknown sparse support.

For continuous diffusion models, \citet{oko2023diffusion} establish nearly minimax-optimal
distribution estimation rates under smoothness assumptions. Related results show how score
estimation and distribution recovery can exploit low-dimensional subspaces \citep{pmlr-v202-chen23o}
and manifold structure \citep{Azangulov2024-uf}. As a discrete counterpart, we consider here support sparsity as
structural assumption. For sequences of length $d$ over a vocabulary of size $K$, the clean state
space $\mathcal X_0$ has $M:=K^d$ elements. An $l$-sparse distribution assigns positive probability
to at most $l$ sequences. We impose no geometric condition on these sequences, no independence
assumption on their coordinates, and no lower bound on their positive probabilities. The relevant
question is therefore whether diffusion-based estimation can adapt to $l$, rather than incur a
statistical cost proportional to $M$.

Our starting point is to regularize the empirical score-entropy objective by its counterpart for
uniformly distributed clean data. The resulting score is induced by the smoothed empirical
distribution
\begin{equation}\label{eq:intro_smoothed_estimator}
    \hat\mu_0^{(\gamma)}
    =\frac{n}{n+\gamma}\,\mu_0^{(n)}
    +\frac{\gamma}{n+\gamma}\,\mathrm{Unif}(\mathcal X_0),
    \qquad \gamma>0,
\end{equation}
where $\mu_0^{(n)}$ is the empirical law of the $n$ training sequences. This distribution has full
support, so it avoids the infinite forward KL loss that occurs when an estimator assigns zero
probability to an unobserved sequence with positive data probability. The estimator in
\eqref{eq:intro_smoothed_estimator} does not depend on the corruption kernel; its forward marginals
and associated scores do.

Smoothing discrete frequency estimates is classical. In particular, \citet{Mourtada_2026} analyzes an
adaptive estimator of the form \eqref{eq:intro_smoothed_estimator} and establishes sparse KL bounds
and minimax lower bounds. Our contribution is to connect this estimation problem to regularized
diffusion training and reverse sampling, and to determine how the smoothing level controls the
resulting statistical error. For masked diffusion, this connection uses the conditional-distribution
representation of the score established by \citet{ou2024radd}.

Our contributions can be summarized as follows:
\begin{enumerate}[label=$\bullet$]
    \item
    In the setting of masked diffusion, we derive a method for regularizing scores via a context-dependent label smoothing. We show that this smoothing is equivalent to regularizing with the score loss of the uniform distribution, which extends it beyond masked diffusion. The optimal conditional predictions that minimize the regularized loss produce the proper score induced by \eqref{eq:intro_smoothed_estimator}.

    \item 
    We then establish expected and high-probability bounds for $\kll(\data\mid\hat\mu_0^{(\gamma)})$ in
    terms of the effective support and effective missing-support quantities of the data distribution, following the approach of \citet{Mourtada_2026}. The bounds allow for data-dependent choices of $\gamma$, and the high-probability event holds simultaneously for every $\gamma\in[1,\min\{n,M\}]$. They therefore describe the effect of regularization beyond exact support sparsity. \tyler{I don't understand this last sentence. Also next sentence not so interesting given next point} We also identify ranges of smoothing levels that attain the sparse minimax rates.
    \luca{do we replace from "They therefore describe..." with " Since sequences with probability much smaller than $1/n$ contribute little to the effective quantities, the bounds remain meaningful beyond exact sparsity, for distributions in which only a small number of sequences carry non-negligible probability."?}

    \item 
    We transfer these bounds to the terminal law $\hat\nu$ of the reverse sampler. For the masked
    diffusion initialized at the fully masked state, we prove the exact identity
    $\hat\nu=\hat\mu_0^{(\gamma)}$. For uniform diffusion, we control the finite-horizon
    initialization error. With the adaptive choice $\gamma=D_n$, where $D_n$ is the number of distinct training
    sequences, both noising procedures attain a sparse minimax rate of order $l\log(eM/l)/n$ in expectation and high probability, without requiring knowledge of $l$ or the support. Instead of depending on the size of the state space, which is exponentially large with respect to sequence length, the rate is instead governed by the sparsity $l$ which can be far smaller in practice.
\end{enumerate}

\paragraph{Relation to previous work.}
Recent analyses of discrete diffusion address several distinct statistical problems.
\citet{srikanth2026discretestatediffusionmodels} derive sample-complexity guarantees for neural
score estimators, accounting for statistical, approximation, optimization, and clipping errors.
\citet{Wakasugi2025-us} obtain state-size-independent bounds by exploiting smooth embeddings and
graph structure; their assumptions differ from support sparsity. \citet{cho2026} derive nearly matching minimax
bounds for score estimation under the score-entropy loss and apply their analysis to both uniform
and masked diffusion. Their bounds depend on the cardinality of the diffusion state space and on
density-ratio bounds; they do not quantify adaptation to the support size of the clean distribution.

Closer to our setting, \citet{zhang2026} derive structure-dependent generalization bounds and
compare the effects of sparsity under masked and uniform corruption. \tyler{Do we not want to mention their rates here? They are worse than our's}\luca{they are different, so maybe it is better not} Such comparisons raise a
further question: does the choice of corruption mechanism change the \emph{minimax} risk over sparse
data distributions? A difference between upper bounds for particular score estimators does not by
itself establish a separation between the optimal risks. We address this question by analyzing a
regularized score estimator and identifying the distribution generated by its reverse dynamics.

\paragraph{Notation.}
For a finite nonempty set $\mathcal I$, let $|\mathcal I|$ denote its cardinality, $\mathcal P(\mathcal I)$ denote the set of probability measures on $\mathcal{I}$, $\mathrm{Unif}(\mathcal I) \in \mathcal P(\mathcal I)$ be its uniform law, $\delta_x \in P(\mathcal I)$ be the point mass at $x \in \mathcal I$.
For $x\in\mathcal X$ and $A\subseteq[d]$, write $x_A:=(x_j)_{j\in A}$ and
$x_{-i}:=x_{[d]\setminus\{i\}}$. For $\mu,\nu\in\mathcal P(\mathcal I)$,
define $ \kll(\mu\mid\nu)
    :=\sum_{x:\mu(x)>0}\mu(x)\log (\mu(x)/\nu(x))$,
when $\mu\ll\nu$, and set $\kll(\mu\mid\nu):=+\infty$ otherwise. %
Given any positive integer $K$, we use the shorthand $[K] =  \{1, ..., K\}$ and we use the notation $a\lesssim b$ to mean $a\le Cb$ for a universal constant $C$.

%% file: sections/preliminaries.tex
\section{Preliminaries}\label{sec:preliminaries}

We consider the problem of density estimation for a target distribution $\data\in\mathcal{P}(\mathcal{X}_0)$ on the finite sequence space, $\mathcal X_0:=\mathcal V^d$, $d \geq 1$. Without loss of generality, we take $\mathcal V:=[K]$ with $K \geq 2$ leading to a cardinality $M := |\mathcal{X}_0| = K^d$ that grows exponentially with $d$.

\paragraph{Discrete diffusion models.}
The goal of discrete diffusion models is to learn a generative model that approximates $\data$ with access to only a finite set of samples $S=(X_0^{(1)},\ldots,X_0^{(n)})\sim\data^{\otimes n}$. This is done by simulating the reversal of a Markov process that starts from the data distribution and converges to a tractable limiting distribution. For convenience of exposition, we define this Markov process on a larger finite set $\mathcal{X} \supseteq \mathcal{X}_0$ and we identify $\data$ with its extension by zero to $\mathcal X$.

The forward process is defined as a continuous-time Markov chain $(X_t)_{t\in[0,T]}$ on $\mathcal{X}$ with time horizon $T > 0$ and rate matrix $(Q_t)_{t\in[0,T]}$ (throughout this work we use the column-generator convention---details in Appendix \ref{appendix_matix_design}). We assume that $Q_t=\sigma_t Q_0$ where $Q_0$ is a fixed conservative generator and $\sigma:[0,\infty)\to[0,\infty)$ is integrable with $
\bar\sigma_t:=\int_0^t \sigma_u\,\mathrm{d}u$ satisfying $0 < \bar\sigma_t < \infty$ for all $t>0$, and $\bar\sigma_t\to\infty$ as $t\to\infty$.
For $0\leq s\leq t \leq T$, $x,y \in\mathcal{X}$, we define the transition matrix $P_{s,t}(x,y):=\mathbb{P}(X_t=x\mid X_s=y)$ which  satisfies,
\begin{equation}\label{eq:transition_kernel_forward}
P_{s,t}=\exp\big((\bar\sigma_t-\bar\sigma_s)Q_0\big)\eqsp.
\end{equation}
In particular, if $\mu_t(x):=\mathbb{P}(X_t=x)$, then $\mu_t$ solves the forward Kolmogorov equation,
\begin{equation}
    \label{definition_forward_process}
\partial_t\mu_t(x)= Q_t \mu_t(x) = \sum_{y\in\mathcal{X}}Q_t(x,y)\mu_t(y) \eqsp.
\end{equation}
We initialize the process at the data distribution $\mu_0 = \data$ so that, by definition, $\mu_t=P_{0,t}\data.$

\paragraph{The reverse process.} The time-reversed process, which we define by $\widetilde X_t:=X_{T-t}$, $t\in[0,T]$, is also a continuous-time Markov chain whose rate matrix $\widetilde Q_t(x, y)$ satisfies the identity, $\mu_{T-t}(y)\widetilde Q_t(x,y) = \mu_{T-t}(x)Q_{T-t}(y,x)$ whenever $x\neq y$, $\mu_{T-t}(y)>0$ (see, e.g., \cite[Theorem~2.8]{conforti2022timereversalmarkovprocesses}).
This motivates the definition of the discrete score,
\begin{equation}\label{definition_true_score}
s^\star(x,y,t):=
\begin{cases}
\displaystyle \frac{\mu_t(x)}{\mu_t(y)},&\mu_t(y)>0,\\[1ex]
0,&\mu_t(y)=0,
\end{cases}
\end{equation}
and hence, for $x\neq y$, the reverse generator can be written as $\widetilde Q_t(x,y)=s^\star(x,y,T-t)Q_{T-t}(y,x)$
with diagonal entries chosen so that the columns sum to zero. Equivalently, the reverse marginal law $\widetilde\mu_t$ solves
\begin{equation}\label{eq:ideal_chain}
\partial_t\widetilde\mu_t=\widetilde Q_t\widetilde\mu_t,
\qquad
\widetilde\mu_0=\mu_T.
\end{equation}
With this initialization, we have that $\widetilde\mu_t=\mu_{T-t}$, and thus, exact sampling from the reverse process would recover the data distribution, $\widetilde\mu_T=\data$.

In practice, the ratios $s^\star$ are unknown and must be estimated by some estimator score $s(x,y,t)$. %
We assume that this score is integrable in $t$ and positive in the triple $(x,y,t)$ where $Q_t(y,x)>0.$ We then define the associated generator $\hat{Q}_t(x,y)$ as $\tilde{Q}_t(x,y)$, but with $s(x,y,T-t)$ in place of $s^\star(x,y,T-t).$
We also do not have access to the initial distribution $\widetilde{\mu}_0 = \mu_T$ so we initialize at some tractable distribution $\mu_\infty\in \mathcal{P}(\mathcal{X})$ which is usually set to be a stationary distribution of the forward process. We define the sampling process on $[0,T)$ as the one which evolves according to 
\begin{equation}\label{sampling_kolmogorov_equation}
\partial_t \nu_t = \hat Q_t\nu_t, \qquad \nu_0=\mu_\infty. 
\end{equation}

We will only use two choices of $Q_0$, corresponding to the uniform and masked noising processes, though we expect that our results may be adapted to other settings; see Appendix~\ref{appendix_matix_design} for details.
 For the uniform process, we set $\mathcal{X} = \mathcal{X}_0$ and consider a rate matrix with identical positive values on the off-diagonal, leading to a process that converges to the uniform measure. Therefore, in this case we set $\mu_\infty=\mathrm{Unif}(\mathcal{X}).$ For the absorbing or masked process, we extend the space $\mathcal X=(\mathcal V\cup\{\mathfrak m\})^d$, where $\mathfrak m\notin\mathcal V$ is a mask token that the forward process converges to. In this case, we set $\mu_\infty=\delta_{(\mathfrak{m},\ldots,\mathfrak{m})}.$

\paragraph{Score-entropy loss.} Following \citep{lou2024discrete}, in order to find the estimator $s$, we utilize the score-entropy loss, which is defined for the weights $w_t(x,y)\geq 0$ by,
\begin{equation}\label{definition_score_entropy_loss}
\mathcal{L}_t(s):=\mathbb{E}_{y\sim \mu_t}\Big[\sum_{x \neq y}w_t(x,y)\Big(s(x,y,t)-s^\star(x,y,t)\log s(x,y,t) + \phi\big(s^\star(x,y,t)\big)\Big)\Big],
\end{equation}
where $\phi(x):=x(\log x -1),$ with $\phi(0)=0$.
Throughout, we take $w_t(x,y)=Q_t(y,x),$ following the literature, for example see \citep{zhang2026}.
In practice, we consider its empirical counterpart,
\begin{equation}\label{eq:loss_empirical}
\hat{\mathcal{L}}_t(s):=\frac{1}{n}\sum_{i=1}^n\mathbb{E}_{y\sim P_{0,t}(\cdot,X_0^{(i)})}\Big[\sum_{x \neq y} Q_t(y,x)\Big(s(x,y,t)- \frac{P_{0,t}(x,X_0^{(i)})}{P_{0,t}(y,X_0^{(i)})}\log s(x,y,t)\Big)\Big],
\end{equation}
which is the score-entropy loss for the empirical law $\mu_0^{(n)}:=n^{-1}\sum_{j=1}^n\delta_{X_0^{(j)}}$.
 
\paragraph{$l-$sparsity and lower bounds.} In this work, we consider the class of $l$-sparse distributions. In particular, for a fixed integer $l$ such that $1\leq l\leq |\mathcal{X}_0|,$ we define the class by,
\begin{equation}\label{definition_class_l_sparse_distributions}
\mathcal{P}_l(\mathcal{X}_0):=\{\mu\in\mathcal{P}(\mathcal{X}_0):|\supp(\mu)|\leq l\}.
\end{equation}
In this paper, we are interested in showing minimax rates for sparse density estimation and so we recall some lower bounds for the problem. This problem has been widely studied  \citep{han2015lowerbound,Kamath15lowebounds}, though we focus on the following recent result of \cite{Mourtada_2026}. Whenever $n, |\mathcal{X}| \geq 110$, $\delta \in (e^{-n}, e^{-2})$, $2 \leq l \leq \min\{n, |\mathcal{X}|/55\}$, for any estimator $\Phi : \mathcal{X}^n \to \mathcal{P}(\mathcal{X})$, there exists a distribution $P \in \mathcal{P}_l(\mathcal{X})$ such that,
\begin{equation}
    \label{eq:lower-bound-minimax}
    \mathbb{P}_P \bigg( \kll(P|\hat P_{n}) \geq \frac{l \log(e|\mathcal{X}|/l) + \log (|\mathcal{X}|) \log (1/\delta)}{320 \,n} \bigg) \geq \delta \, ,
\end{equation}
where $\hat P_n = \Phi (X_1, \dots, X_n)$. From this, we immediately obtain bounds in expectation,
\begin{equation*}
    \E[\kll(P|\hat P_{n})] \geq \frac{l \log(e|\mathcal{X}|/l)}{320 e^2 n}.
\end{equation*}

Further results from  \citep{Mourtada_2026} can be found in Appendix~\ref{appendix_useful_results}.
\tyler{Maybe prove expectation bound somewhere}

%% file: sections/statistical_rate_first_part.tex
\section{Statistical rates for masked diffusion}\label{sec:statitical_part}
The goal of our analysis is to derive bounds for $\kll(\data\mid \hat{\nu})$ where we define $\hat\nu:=\lim_{u\to T}\nu_u,$ and $\nu_u$ is the solution of \eqref{sampling_kolmogorov_equation}. In order to do this, we derive a particular choice of estimator $s(x,y,t)$, then we obtain upper bounds on $\kll(\data\mid \hat{\nu})$ as a function of a tuning parameter $\gamma$ which we then compare with minimax lower bounds for the sparse distribution class. With this, we find two intervals for $\gamma$, one for the bound in expectation and one for the bound in high-probability, for which we reach the minimax optimal rates for the sparse distribution class.

Throughout this section, we denote by $S$ our i.i.d. dataset, and denote by $\mathbb{E}_S$ and $\mathbb{P}_S$ the expectation and the probability with respect to the dataset. In the following, we begin with a problem of minimizing the empirical risk, which is a standard object of study in the training of discrete diffusion models. However, in the present setting, we cannot directly use the minimizer as it produces a score with infinite KL divergence. Thus, we look for non-parametric methods for regularizing the empirical problem.

\subsection{Label smoothing for discrete scores}\label{sec:entropy}

We first motivate our approach to regularizing the score by relating the score matching problem to classification. We define a \emph{context} to be a triple $c=(A,v,i)$ where $A\subseteq[d]$, $v\in\mathcal V^A$ and $i\in[d]\setminus A$, and for a partially masked state $y$ with masked position $i$ we write $c(y,i):=(A(y),y_{A(y)},i)$. In \citet{ou2024radd}, it was first shown that the score function $s^\star$ for the masked diffusion admits the factorization,
\begin{equation}\label{eq:ou_factorization}
    s^\star(x,y,t)=\kappa_t\,\psi^\star_{c(y, i)}(x_i), \qquad \kappa_t:=\frac{\alpha_t}{1-\alpha_t},
\end{equation}
where $\psi^\star_c(a):=\data(X_i=a\mid X_A=v)$ is the conditional law of the clean token  (see Proposition~\ref{prop:ou}). Thus, the time dependence of the score can be obtained in closed-form, and the only quantity to learn is $\psi^\star_c$, which can be viewed as a \emph{classifier} that predicts the token at position $i$ from the unmasked tokens $v$. We thus parametrize the score estimator as $s(x, y, t)=\kappa_t\,\psi_{c(y, i)}(x_i)$ with $\psi_{c(y, i)} \in\mathcal P(\mathcal V)$ which is also a parameterization used in practice, where the score network outputs a softmax over $\mathcal V$ at each masked position.

To take this analogue further, we can also describe the score-entropy loss via the cross-entropy. Let $n_c :=|\{k: X^{(k)}_{0,A}=v\}|$, and let $\psi^{(n)}_c$ be the conditional distribution for the empirical data distribution, $\psi^{(n)}_c(a):=\mu_0^{(n)}(X_i=a\mid X_A=v)$. Plugging \eqref{eq:ou_factorization} into \eqref{eq:loss_empirical}, it can be easily shown that
\begin{equation}\label{eq:loss_classifier}
    \int_0^T\hat{\mathcal L}_t(s)\,dt=\sum_c \omega_{|A|} \,\frac{n_c}{n}\,H\big(\psi^{(n)}_c, \psi_c \big) + \mathrm{const},
\end{equation}
where $H(p,q):=-\sum_a p(a)\log q(a)$ is the cross-entropy and the weights $\omega_k>0$ depend only on $k$ and $T$. The minimizer of \eqref{eq:loss_classifier} is the empirical conditional $\psi^{(n)}_c$ on every observed context, while contexts with $n_c=0$ are left unconstrained.
\tyler{TODO: include explicit proof somewhere}

By viewing $\psi_c$ as a classifier trained with the cross-entropy loss, a natural choice of regularization is label smoothing \citep{szegedy2016rethinking}. In classification, label smoothing adds the penalty $H(\Unif, \cdot)$ to the loss to reduce confidence in predictions and prevent overfitting. We thus add a label smoothing penalty to each classifier with a context-dependent parameter $\varepsilon_c \in[0, 1]$:
\begin{equation}\label{eq:regularized_objective}
    (1 - \varepsilon_c) H\big(\psi^{(n)}_c, \psi_c \big) + \varepsilon_c  H\big(\Unif(\mathcal V), \psi_c \big).
\end{equation}
By linearity of the cross-entropy in the first argument, we have that this is minimized by, $\hat\psi_c = (1-\varepsilon_c)\,\psi^{(n)}_c+\varepsilon_c\,\Unif(\mathcal V)$.

Note that for a generic choice of $\varepsilon_c$, it is not clear that the score function resulting from label smoothing will necessarily produce a \emph{proper} score: that is, a score $s$ that relates to a diffusion for a probability measure $\mu_s$. In the following lemma, we show that there is a class of such choices of $\varepsilon_c$ and that this class produces a smoothing of the empirical measure.

\begin{lemma}\label{lem:proper_smoothing}
Suppose that for every $i\in[d]$, there exists $b_i\in\mathcal V$ such that $\data(X_i=b_i)=0$, and that $\varepsilon_c=1$ for every context $c=(A,v,i)$ with $n_c=0$. Then, the score function $s(x,y,t):=\kappa_t\hat\psi_{c(y,i)}(x_i)$ is proper if and only if there exists $\gamma\in(0,\infty]$, not depending on $c$, such that
\begin{equation}\label{eq:eps}
\varepsilon_c=\frac{\gamma K^{-|A|}}{\gamma K^{-|A|}+n_c},\qquad\text{for every context }c=(A,v,i),
\end{equation}
where $\varepsilon_c:=1$ if $\gamma=\infty$. In particular, $s$ is the proper score induced by the terminal law, $\hat\mu_0^{(\gamma)}$.
\end{lemma}

 We see that the smoothing is as in \citet{szegedy2016rethinking}, but where the level of smoothing $\varepsilon_c$ decreases with the number of observations $n_c$ of the context rather than being fixed. In particular, contexts that are observed frequently are barely smoothed, and contexts that do not appear in the sample are predicted uniformly. For this choice of $\varepsilon_c$, we find an elegant relationship between label smoothing and the score entropy loss relative to the uniform sequence distribution, ${\mathcal L}^U$. In particular, for the mixed loss $\hat{\mathcal L}^M_t(s) = (1-\tfrac{\gamma}{n+\gamma})\,\hat{\mathcal L}_t(s) +\tfrac{\gamma}{n+\gamma}\,{\mathcal L}^U_t(s)$, we have,
\begin{align*}\label{eq:loss_mixture}
    &\int_0^T \hat{\mathcal{L}}^M_t(s)\,dt =\sum_c \omega_{|A|}\,\tfrac{n_c+\gamma K^{-|A|}}{n+\gamma} \Big ( (1-\varepsilon_c) H\big(\psi^{(n)}_c, \psi_c \big) + \varepsilon_c H\big(\Unif(\mathcal V), \psi_c \big) \Big ) + \mathrm{const},
\end{align*}
as shown in Lemma~\ref{lem:loss_contexts}.

For the statistical analysis pursued in this work, we take this class of proper label-smoothed score functions as the estimator which we analyze. In particular,
\begin{equation}\label{eq:conditional_is_smoothed}
    \hat s^{(\gamma)}(x,y,t):=\kappa_t \big ( (1-\varepsilon_c)\,\psi^{(n)}_c(x_i)+\varepsilon_c\,\Unif(\mathcal V) \big ),
    \qquad
    \varepsilon_c=\frac{\gamma K^{-|A|}}{\gamma K^{-|A|}+n_c},
\end{equation}
In the sections to follow, we analyze this estimator and show that for an appropriately tuned $\gamma$, we obtain minimax optimal rates in sparse settings.

\subsection{Upper bounds}\label{sec:upper_bounds}

We now proceed with the statistical analysis of the estimator introduced above. To begin with, we introduce two fundamental data-dependent quantities that we borrow from \citep{Mourtada_2026} and use in our analysis.
\begin{definition}\label{def:effective_support_missing}
Given a probability distribution $\mu$, let us define the effective support size $k_{\mathrm{eff}}(\mu,n)$  of $\mu$ at a sample size $n$ as
\begin{equation}
k_{\mathrm{eff}}(\mu,n):=\sum_{x\in\mathcal{X}_0} \min\{1,n \mu(x)\}=|\{x\in \mathcal{X}_0:\mu(x)\geq 1/n\}|+\sum_{x:\mu(x)<1/n} n\mu(x).
\end{equation}
Similarly, we define the effective
missing support size $k_{\mathrm{eff}}^o(\mu,n)$ of $\mu$ with sample size $n$  as
\begin{equation}
k_{\mathrm{eff}}^o(\mu,n):=\sum_{x\in \mathcal{X}_0}\min\{\exp(1-n\mu(x)),n\mu(x)\}=\sum_{x:\mu(x)\geq 1/n}\exp(1-n\mu(x))+\sum_{x:\mu(x)<1/n} n\mu(x).
\end{equation}
\end{definition}

The quantity $k_{\mathrm{eff}}(\mu,n)/n$ relates to error in the empirical measure due to low probability sequences that will likely be missing from the dataset. We note that $k_{\mathrm{eff}}(\mu,n)$ is increasing with $n$ while $k_{\mathrm{eff}}(\mu,n)/n$ is decreasing. Also, $k_{\mathrm{eff}}(\mu,1) = 1$, $k_{\mathrm{eff}}(\mu,n)\leq l$ if $\mu\in\mathcal{P}_l$, and $k_{\mathrm{eff}}(\mu,n)$ converges to
$k_{\mathrm{eff}}(\mu,\infty) = \supp(\mu)$ as $n\to\infty$.

We now provide two important lemmas which upper bound the $\kll$ between $\data$ and $\hat\mu_0^{(\gamma)}$. Throughout this section, we require that $\gamma$ is a measurable map from $\mathcal{X}_0^n$ to $[1,\min\{n,K^d\}]$. For the sake of brevity, we denote by $\hat\mu_0, k_{\mathrm{eff}},L_\gamma$ the objects $\hat\mu_0^{(\gamma(S))}, k_{\mathrm{eff}}(\data,n)$ and $\log \frac{eK^d}{\gamma}$, respectively.
\begin{lemma}\label{lem:upper_bound_expectation_mask}
If $n\ge8$, then for every deterministic $\gamma_0\in[1,\min\{n,K^d\}]$,
\begin{equation}\label{eq:theorem_upper_bound_expectation}
 \mathbb E_S\,\kll(\data|\hat\mu_0)\ \le\ C\frac{k_{\mathrm{eff}}+\,\mathbb E_S[\gamma]+\,k_{\mathrm{eff}}^o(\data,n/8)\,L_{\gamma_0}}{n}\ +\ \log(\gamma_0)\,\mathbb P_S\big(\gamma<\gamma_0\big),
\end{equation}
for some universal constant $C > 0$.
\end{lemma}
\begin{lemma}\label{lem:upper_bound_probability_mask}
If $n\ge13$, $K^d\ge3$ and $\delta\in(e^{-n/6},e^{-2})$, there is an event
$A$ with $\mathbb P_S(A)\ge1-12\delta$, defined in terms of $S$ and independent of $\gamma$, on which, simultaneously for every
 $\gamma\in[1,\min\{n,K^d\}]$,
 \[
 \kll\big(\data\,\big|\,\hat\mu_0^{(\gamma)}\big)\ \le\
C\frac{k_{\mathrm{eff}}+\gamma+\,k_{\mathrm{eff}}^o(\data,n/112)\,L_\gamma+(1+\,L_\gamma)\log\tfrac1\delta+\log\tfrac1\delta\log\log\tfrac1\delta}{n},
 \]
for some universal constant $C > 0$.
\end{lemma}

We note that if $\gamma_0=1,$ the term $\log(\gamma_0)\mathbb{P}(\gamma<\gamma_0)=0,$ but the term $L_{\gamma_0}$ takes the maximum value of $1+d\log K$. Thus, we have a trade off, which will be optimized depending on the setting.%

The lemmas above provide bounds, both in expectation and in high probability, on the KL divergence between $\data$ and the smoothed empirical distribution $\hat\mu_0^{(\gamma)}$. However, to analyze the diffusion model, we must instead consider $\nu_t$. We do this via the following lemma.

\begin{lemma}\label{lem:sampler_mask}
Let $\gamma>0$, let $\hat\mu_0^{(\gamma)}$ be as in \eqref{eq:intro_smoothed_estimator} and $\hat s = \hat s^{(\gamma)}$ as in \eqref{eq:conditional_is_smoothed}. Then \eqref{sampling_kolmogorov_equation} with score $\hat s$ and initial law $\delta_{(m,\dots,m)}$ has a unique absolutely continuous solution $(\nu_u)_{u\in[0,T)}$, the limit $\hat\nu:=\lim_{u\uparrow T}\nu_u$ exists, and $\hat\nu=\hat\mu_0^{(\gamma)}$, where $\hat\mu_0^{(\gamma)}$ is extended by zero to $\mathcal X$: $\hat\nu(x)=\hat\mu_0^{(\gamma)}(x)$ for $x\in\mathcal X_0$ and $\hat\nu(x)=0$ for $x\in\mathcal X\setminus\mathcal X_0$. In particular, for every $\data\in\mathcal P(\mathcal X_0)$,
\[
\kll(\data\,|\,\hat\nu)=\kll(\data\,|\,\hat\mu_0^{(\gamma)}).
\]
\end{lemma}

Putting together the lemmas above, we obtain the following result.
\begin{theorem}\label{thm:upper_bounds_mask}
Let $\gamma$ be a measurable map from $\mathcal{X}_0^n$ to $[1,\min\{n,K^d\}]$, and denote by $k_{\mathrm{eff}},L_\gamma$ the quantities $k_{\mathrm{eff}}(\data,n)$ and $\log \frac{eK^d}{\gamma}$, respectively. Then there exists a universal constant $C$ such that the following hold:
\begin{enumerate}[leftmargin=2em]
    \item If $n\ge8$, then for every deterministic $\gamma_0\in[1,\min\{n,K^d\}]$,
\begin{equation}
\mathbb{E}_S\kll(\data\mid\hat\nu)\leq  C\frac{k_{\mathrm{eff}}+\mathbb E_S[\gamma]+k_{\mathrm{eff}}^o(\data,n/8) L_{\gamma_0}}{n} \, + \, \log(\gamma_0) \, \mathbb P_S\big(\gamma<\gamma_0\big). 
\end{equation}
\item If $n\ge13$, $K^d\ge3$ and $\delta\in(e^{-n/6},e^{-2})$, there is an event
$A$ with $\mathbb P_S(A)\ge1-12\delta$ on which, simultaneously for every
 $\gamma\in[1,\min\{n,K^d\}]$,
 \begin{equation}
 \kll\big(\data\,\big|\,\hat\nu\big) \leq
C\frac{k_{\mathrm{eff}}+\gamma+k_{\mathrm{eff}}^o(\data,\tfrac{n}{112}) L_\gamma}{n} + \frac{(1+L_\gamma)\log\tfrac1\delta+\log\tfrac1\delta\log\log\tfrac1\delta}{n}.
 \end{equation}
\end{enumerate}
\end{theorem}
The theorem above provides upper bounds of the error between $\data$ and the limit of the sampling process \eqref{sampling_kolmogorov_equation}. We highlight that this limit is not a mere theoretical object, but is the terminal law of an implementable algorithm based on the Gillespie method \citep{gillespie}.

\subsection{Sparse setting and minimax bounds}\label{sec:sparse_case}

We now focus on the motivating setting where $\data$ is sparse, i.e. $\data\in \mathcal{P}_l(\mathcal{X}_0)$. We start by noticing that $k_{\mathrm{eff}}^o(\data,n),k_{\mathrm{eff}}(\data,n)\leq l$ for every $n$, and so, under no additional assumptions on $\gamma$, Theorem~\ref{thm:upper_bounds_mask} produces the following corollary. 
\begin{corollary}[Upper bounds in the sparse setting ]\label{cor:sparse_mask}
Let us assume the hypothesis of Lemma~\ref{lem:sampler_mask}. Moreover, let us consider the setting of Lemma~\ref{lem:upper_bound_expectation_mask} with $\data\in \mathcal{P}_l(\mathcal{X}_0).$
Then, the bound becomes
    \begin{equation}
\mathbb{E}_S\kll(\data\mid \hat\nu)\leq C\frac{l+\mathbb{E}[\gamma]+lL_{\gamma_0}}{n}+\log(\gamma_0)\mathbb{P}(\gamma < \gamma_0).
\end{equation}
Similarly, in the setting of Lemma~\ref{lem:upper_bound_probability_mask}, with $\data\in \mathcal{P}_l(\mathcal{X}_0)$ the bound becomes
\begin{equation}
    \kll(\data\mid \hat\nu)\leq C\frac{l+\gamma + lL_\gamma+(1+L_\gamma)\log(1/\delta)+\log (1/\delta)\log\log (1/\delta)}{n}.
\end{equation}
\end{corollary}

For deterministic $\gamma$, Corollary \ref{cor:sparse_mask} produces a bound with rate $\frac{1}{n} (l\log{\frac{eK^d}{\gamma}} + \gamma)$. Thus, even with the generic choice of $\gamma = 1$, we obtain $O(\frac{l d \log(K)}{n})$ which no longer necessarily scales exponentially with dimension. However, the rates are still sub-optimal when $l$ is large, and to achieve minimax optimal rates of $O(\frac{l \log(eK^d/l)}{n})$, we effectively require that $\gamma \sim l$. In the following corollary, we show that it is in fact sufficient for $\gamma$ to be \emph{concentrated} around $l$ to obtain such rates.
\tyler{Why were the assumptions taken out of the corollary?}\luca{three reasons: i don't want the corollary too big, coherence with the lenght of the other non main results, and limit page problems}
Throughout, we require that $n,K^d\ge110$ and $2\le l\le\min\{n,\lfloor K^d/55\rfloor\}$; we write $\Lambda_l:=\log(eK^d/l)$ and $\Gamma_\delta:=l\Lambda_l+\log(K^d)\log(1/\delta)$, so that $\Gamma_1=l\Lambda_l$. For a measurable $\gamma:\mathcal X_0^n\to[1,\min\{n,K^d\}]$ we denote by $\hat\mu_0=\hat\mu_0^{(\gamma)}$ the estimator (1), by $\hat s$ the score, and by $\hat\nu$ the terminal law of \eqref{sampling_kolmogorov_equation} driven by $\hat s$.

\begin{corollary}\label{prop:minimax_mask_prob}
Fix constants $0<c_0\le1\le C_0$. There exist a universal constant $c>0$ and a constant $C<\infty$ depending only on $(c_0,C_0)$ such that the following hold.
\begin{enumerate}[label=(\roman*)]
\item If $\gamma$ takes values in $[\max\{1,c_0l\},\,\min\{n,K^d,C_0\,l\Lambda_l\}]$, then for every $\data\in\mathcal P_l(\mathcal X_0)$,
\[
\mathbb{E}_S\,\KL(\data\,|\,\hat\nu)\le C\,\frac{l\Lambda_l}{n}.
\]
\item Let $\delta\in(e^{-\min\{n/6,K^d\}},e^{-2})$. If $\gamma$ takes values in $[\max\{1,c_0l\},\,\min\{n,K^d,C_0\Gamma_\delta\}]$, then for every $\data\in\mathcal P_l(\mathcal X_0)$, with probability at least $1-12\delta$,
\[
\kll(\data\,|\,\hat\nu)\le C\,\frac{\Gamma_\delta}{n}.
\]
\end{enumerate}
\end{corollary}

We observe that the interval above depends on the sparsity parameter $l$, which is unknown in general. We find that a suitable estimator that reaches similar rates is given by,
\begin{equation}\label{eq:definition_D_n}
    D_n:=|\supp\mu_0^{(n)}|=\#\{x\in \mathcal{X}_0: N_0(x)>0\},
\end{equation}
which denotes the number of different states observed in the sample at time $t=0.$ Using it, we can define an implementable, adaptive estimator, which reaches the optimal rates described above. The result is summarized in the Appendix \ref{appendix_final} (Corollary~\ref{cor:adaptive_gamma}).

%% file: sections/statistical_rates.tex
\section{Statistical rates for other diffusions}\label{sec:statistical_rates:general}

We now extend outside of the masked setting to more general diffusions. For the masked diffusion, the factorization in \eqref{eq:ou_factorization} expresses the score as a classifier for the masked token multiplied by an explicit time-dependent factor, and it is this structure that motivated our use of label smoothing. For a general generator $Q_0$, no such factorization is available. We show, however, that the mixed score entropy loss $\mathcal{L}^M$ of Section~\ref{sec:entropy} still admits an explicit minimizer, which smooths the empirical score in a manner analogous to label smoothing. We use this to extend our statistical guarantees to uniform corruption setting.

By linearity of the score entropy loss with respect to the data distribution, the mixed loss $\mathcal{L}^M$ can be expressed as a score entropy loss for the measure $\hat{\mu}_0^{(\gamma)}$ defined in \eqref{eq:intro_smoothed_estimator}. Thus, via the equivalence of the score entropy loss and the score matching loss, $\mathcal{L}^M$ is uniquely minimized by the score,
\begin{equation}\label{eq:hat_s_uniform}
    \hat{s}(x, y, t) := \frac{\hat{\mu}_t(x)}{\hat{\mu}_t(y)} = \frac{P_{0, t} \hat{\mu}_0^{(\gamma)}(x)}{P_{0, t} \hat{\mu}_0^{(\gamma)}(y)}.
\end{equation}
Writing $\mu^{(n)}_t:=P_{0,t}\mu^{(n)}_0$ for the noised empirical measure and $s^{\star,(n)}$ for its score, we have $\hat\mu_t=\frac{n}{n+\gamma}\mu^{(n)}_t+\frac{\gamma}{n+\gamma}\,\mu^U_t $, and dividing $\hat\mu_t(x)$ by $\hat\mu_t(y)$ gives
\begin{equation}\label{eq:smoothed_score}
    \hat{s}(x,y,t) = \big(1-\varepsilon_t(y)\big)\, s^{\star,(n)}(x,y,t)+\varepsilon_t(y)\,s^\star_U(x,y,t), \quad \varepsilon_t(y):=\frac{\gamma\,\mu^U_t(y)}{\gamma\,\mu^U_t(y)+n\mu^{(n)}_t(y)}.
\end{equation}
In the masked case, \eqref{eq:smoothed_score} recovers exactly the label smoothing of Section~\ref{sec:entropy}.

\begin{figure}[t]
    \centering
    \includegraphics[width=\linewidth]{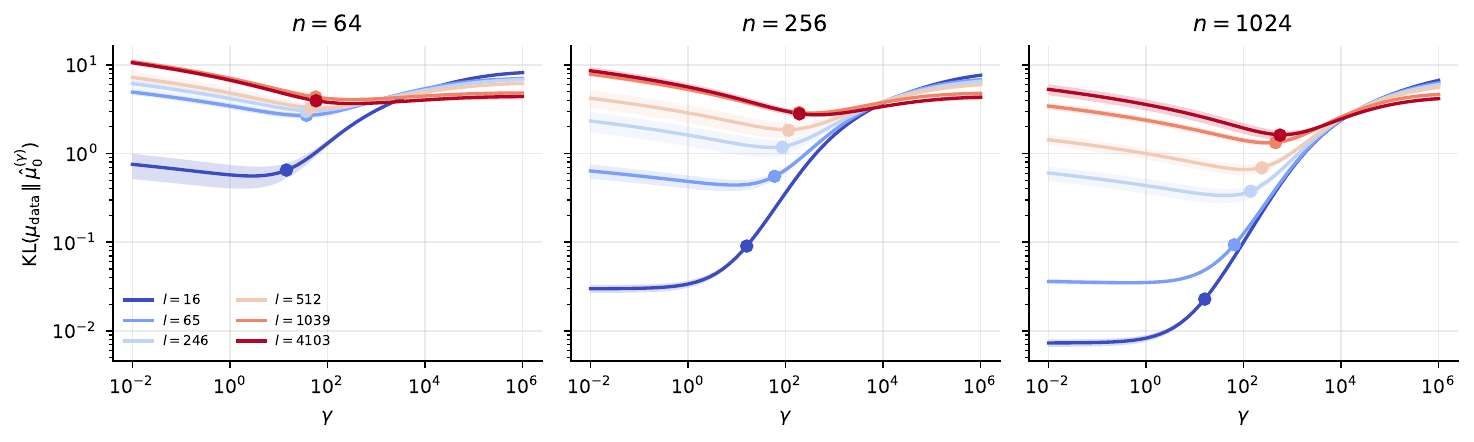}
    \caption{Label-smoothed score on restricted Markov chain dataset. $\kll(\data\mid\hat\nu)$ against $\gamma$, for $n\in\{64,256,1024\}$ for different values of $l$ (blue to red). Dots mark the adaptive choice $\gamma=D_n$.}
    \label{fig:closed_form_1}
\end{figure}

\begin{figure}[t]
    \centering
    \includegraphics[width=0.5\linewidth]{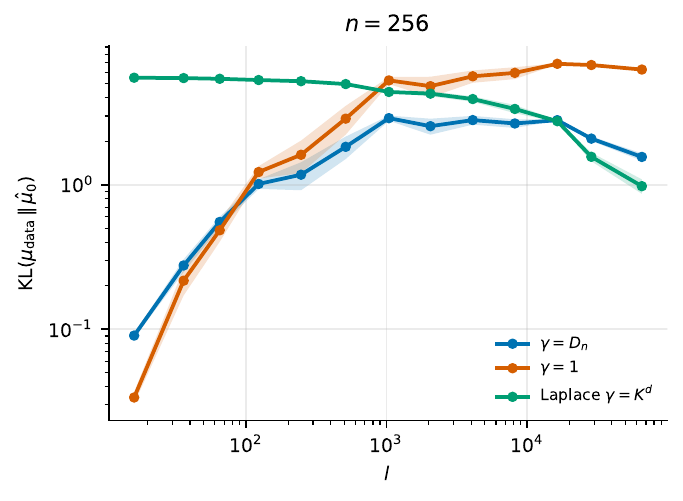}
    \caption{Label-smoothed score on restricted Markov chain dataset as in Figure \ref{fig:closed_form_1}. $\kll(\data\mid\hat\nu)$ against $l$ for different choices of $\gamma$ demonstrating sparsity-adaptive behavior of small $\gamma$.}
    \label{fig:closed_form_1}
\end{figure}

\subsection{Uniform diffusion}

We now consider the case of the uniform diffusion, for which we can obtain a neat expression of the mixed score-entropy loss as well as statistical rates for the limiting terminal distribution. In this case, $\mu^U_t=\Unif(\mathcal X_0)$ for every $t$, and so $s^\star_U \equiv 1$. Thus, the uniform score entropy loss takes the simple form,
\begin{equation*}
    \mathcal L^U_t(s)=\frac{1}{K^d}\sum_{y\in\mathcal X_0}\sum_{x\neq y}Q_t(y,x)\big(s(x,y,t)-\log s(x,y,t)-1\big).
\end{equation*}
Furthermore, $\varepsilon_t(y)=\gamma/(\gamma+nK^d\mu^{(n)}_t(y))$ and the measure $\mu^{(n)}_t(y)$ has a closed-form expression:
\begin{equation*}
    \mu^{(n)}_t(y) = \frac{K^{-d}}{n}\sum_{k=1}^n\big(1+(K-1)\alpha_t\big)^{d-H_k(y)}(1-\alpha_t)^{H_k(y)}=:w_t(y),
\end{equation*}
where $H_k(y):=\#\{j\in[d]:y_j\neq X^{(k)}_{0,j}\}$ is the Hamming distance between $y$ and the $k$-th sample.

Turning to statistical rates, Lemmas~\ref{lem:upper_bound_expectation_mask} and \ref{lem:upper_bound_probability_mask} apply without modification, since they concern only the estimator $\hat\mu_0^{(\gamma)}$. The difference lies in the sampling step. For masked diffusion, the event that all tokens are masked at time $T$ is independent of $X_0$, so that initializing the sampler at the fully masked state recovers $\hat\mu_0^{(\gamma)}$ exactly for every $T$ (Lemma~\ref{lem:sampler_mask}). Uniform diffusion admits no such state, and initializing the sampler at $\Unif(\mathcal X_0)\neq\hat\mu_T$ incurs an additional error, which is controlled in the following lemma.

\begin{lemma}[Relative entropy of the sampler output]\label{lem:sampler_uniform}
Let $\gamma> 0$, let $\hat\mu_0^{(\gamma)}$ be as in \eqref{eq:intro_smoothed_estimator} and $\hat s$ as in \eqref{eq:hat_s_uniform}, and assume $\bar\sigma_t>0$ for every
$t\in(0,T]$. Let $(\nu_u)_u$ be the solution of \eqref{sampling_kolmogorov_equation} with score $\hat s$ and
$\hat\nu:=\lim_{u\uparrow T}\nu_u$ its terminal law. If $\bar\sigma_T\ge\log(dKn)$, then for every $\data\in\mathcal P(\mathcal X_0)$, the solution exists in all $[0,T]$, with $\hat\nu=\nu_T$ and $\kll(\data|\hat\nu)\le\kll(\data|\hat\mu_0^{(\gamma)})+\frac1n$.
\end{lemma}
The initialization error thus contributes an additional term of order $1/n$. Since $k_{\mathrm{eff}}\ge1$, this term is absorbed into $k_{\mathrm{eff}}/n$ up to a universal constant. As a result, uniform diffusion attains the same bounds as masked diffusion, at the cost of a logarithmic lower bound on the time horizon that is not required in the masked case. In the following theorem, we obtain end-to-end bounds for the limiting terminal law. 

\begin{theorem}[Uniform diffusion]\label{thm:uniform}
Let $\gamma$ be a measurable map from $\mathcal{X}_0^n$ to $[1,\min\{n,K^d\}]$ and consider the uniform diffusion with $\bar\sigma_T\geq\log(dKn)$ and terminal law $\hat\nu:=\lim_{u\uparrow T}\nu_u$. Then, both statements of Theorem~\ref{thm:upper_bounds_mask} hold for some different universal constant $C$.
\end{theorem}

In particular, under the hypotheses of the theorem above, the minimax rates obtained for the masked setting in Corollary~\ref{cor:sparse_mask} and Corollary~\ref{prop:minimax_mask_prob} transfer to the uniform setting also.

%% file: sections/experiments.tex
\begin{figure}[t]
    \centering
    \includegraphics[width=0.95\linewidth]{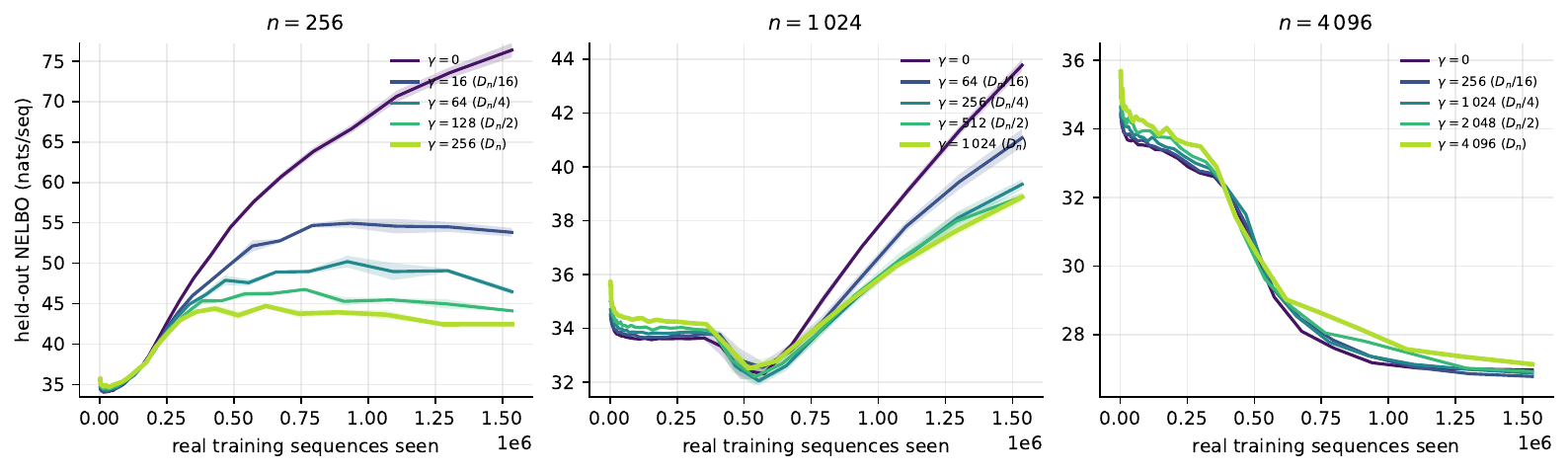}
    \caption{Transformer model trained with regularized score-entropy loss, $\mathcal{L}^M$ on the
text8 $12$-gram target ($K=27$, $d=12$). Training trajectories are plotted for different values of $\gamma$.}
    \label{fig:closed_form}
\end{figure}

\section{Experiments}\label{sec:experiments}

We now provide some experiments that empirically evaluate our theoretical results. This also allows us to inspect the potential practical utility of two objects introduced in Section~\ref{sec:entropy}: the label-smoothed estimator, which requires no training, and the mixed loss $\mathcal{L}^M_t$, which we use to regularize a score network during training. %

\paragraph{Exact label smoothing.} We first analyze the effect of label smoothing by computing the closed form estimator $\hat\psi_c$ directly from dataset counts. As targets, we take the stationary laws of sparse Markov chains on $\mathcal V=[4]$ with $d=8$, so that $K^d=65{,}536$. Restricting the allowed transitions yields support sizes $l$ ranging from $16$ to $4096$. The masked sampler returns exactly $\hat\mu_0^{(\gamma)}$ and so $\kll(\data\mid\nu)$ can be computed in closed form. Figure~\ref{fig:closed_form_1} shows that, for every $l$ and $n$, the error forms a U-shape in $\gamma$ and that the adaptive choice of $\gamma=D_n$ lands close to but not at its minimum, without any direct knowledge of $l$.

\paragraph{Learning with the mixed loss.} We now inspect the utility of the mixed score-entropy loss $L^M_t$ for training, which can be
approximated by replacing each training sequence by a uniformly random one with probability
$p = \gamma/(n+\gamma)$. As a target we take the empirical $12$-gram law of text8 ($K=27$,
$d=12$), drawing training and held-out sequences from disjoint halves of the corpus. At this
length essentially every drawn sequence is distinct ($D_n/n \ge 0.994$), so this experiment
probes the practical utility of the regulariser rather than the sparsity bound, which is vacuous when $k_{\mathrm{eff}} = n$. We anchor the grid to the
data-driven value, $\gamma \in \{0, D_n/16, D_n/4, D_n/2, D_n\}$, so that the prescription
$\gamma = D_n$ lies on the grid at every $n$, and give every arm
$\lceil 6000/(1-p) \rceil$ iterations so that all arms see the same expected number of real
training sequences.

Our results suggest that there may be some practical utility to utilizing label smoothing in the context of discrete diffusion. However, validating this will require a detailed ablation in more practical high dimensional settings, which we leave for future work.

%% file: sections/appendix.tex
\section{Background}
\subsection{Background on the Process}\label{appendix_matix_design}

Throughout, as in Section~\ref{sec:preliminaries}, we use the column-generator convention, which means that for $x\neq y$, $Q_t(x,y)\geq 0$ is the instantaneous rate of a jump from $y$ to $x$, and $Q_t(y,y)=-\sum_{x\neq y}Q_t(x,y)$. Also, we defined $P_{0,t}(x,y)=\mathbb P(X_t=x\mid X_0=y)$, so we can rewrite $\mu_t=P_{0,t}\mu_0$. Since the infinitesimal generator can be rewritten as
$Q_t=\sigma_tQ_0,$ with $Q_0$ a fixed conservative, the generators
commute across times and
\begin{equation}
\label{eq:propagator}
P_{0,t}=\exp\Bigl(\int_0^{t}Q_u\,du\Bigr)=e^{\bar\sigma_tQ_0},
\qquad \bar\sigma_t=\int_0^t\sigma_u\,du,
\end{equation}
with the semigroup property $P_{0,t_3}=e^{(\bar\sigma_{t_3}
-\bar\sigma_{t_2})Q_0}P_{0,t_2}$ for $t_2\le t_3$. 

In both processes below, the matrix $Q_0$
acts independently on the $d$ coordinates, so we can connect its sequence version $Q_0^{\mathrm{seq}}$ with the token one $Q_0^{\mathrm{tok}}$ as 
$Q_0^{\mathrm{seq}}=\sum_{i=1}^d I^{\otimes(i-1)}\otimes
Q_0^{\mathrm{tok}}\otimes I^{\otimes(d-i)}$. In particular,
$P_{0,t}^{\mathrm{seq}}=(P_{0,t}^{\mathrm{tok}})^{\otimes d}$: the kernels
factorize coordinatewise, and it suffices to compute the single-token
propagator. Finally, both token generators are of the form $Q_0^{\mathrm{tok}}
=\Pi-I$ with $\Pi$ a projection (thus it satisfies $\Pi^2=\Pi$), and $I_K$ the identity matrix $K\times K$, so
\begin{equation}
\label{eq:projection-exp}
e^{\bar\sigma_t(\Pi-I)}
=\Pi+e^{-\bar\sigma_t}(I-\Pi)
=\alpha_t\,I+(1-\alpha_t)\,\Pi,
\qquad \alpha_t:=e^{-\bar\sigma_t}.
\end{equation}

\paragraph{Uniform process.}
On the token alphabet $\mathcal V=[K]$, take
$\Pi:=\tfrac1K\mathbf1\mathbf1^\top$ and
$Q_0^{\mathrm{unif}}:=\Pi-I_K$: every token jumps to each of the other
$K-1$ values at rate $1/K$. By \eqref{eq:projection-exp} and
tensorization,
\[
P_{0,t}^{\mathrm{unif}}(x, x_0)
=\prod_{i=1}^{d}\Bigl[\alpha_t\mathbf 1\{x_i=x_{0,i}\}
+\tfrac{1-\alpha_t}{K}\Bigr].
\]
 The
off-diagonal rates of $Q_0^{\mathrm{seq}}$ equal $1/K$, so
$Q_0^{\max}=1/K$, and the total exit rate of any state is
$d(K-1)/K\le d$. The uniform law on $\mathcal V^d$ is stationary
(each column of $\Pi-I_K$ sums to zero against the uniform vector).

\paragraph{Canonical absorbing process.}
On the token alphabet $\mathcal V\cup\{\mathfrak m\}$ (size $K+1$),
take $\Pi:=e_{\mathfrak m}\mathbf1^\top$ and
$Q_0^{\mathrm{mask}}:=\Pi-I_{K+1}$: every real token jumps to
$\mathfrak m$ at unit rate, and $\mathfrak m$ is absorbing. By
\eqref{eq:projection-exp} and tensorization,
\[
P_{0,t}^{\mathrm{mask}}(x, x_0)
=\prod_{i=1}^{d}\Bigl[\alpha_t\mathbf 1\{x_i=x_{0,i}\}
+(1-\alpha_t)\mathbf 1\{x_i=\mathfrak m\}\Bigr].
\]
The
point mass on the fully masked sequence is stationary. The total exit
rate of a state $y$ equals its number of unmasked coordinates, so
the total exit rate of any state is less or equal than $d$. Finally, for a state $y$ with masked coordinate
$i\in M(y)$ and any $a\in\mathcal V$, the sequence-level rate from
$(y_{-i},a)$ to $y$ is the token-level rate $a\to\mathfrak m$, namely
$Q_0\bigl(y,(y_{-i},a)\bigr)=1$.
\subsection{Useful results}\label{appendix_useful_results}

We start by giving two results that are going to be useful in the following, when we are going to need lower bounds on the class of sparse distributions.

\begin{proposition}[\citep{Mourtada_2026}, Section 4.1, Proposition 1]\label{prop:mourtada_lower_bound}
Let $n$ be the number of samples $(X^{(1)},\ldots,X^{(n)})$, and $n, |\mathcal{X}| \geq 2$. Let us assume $P\in \mathcal{P}_l(\mathcal{X}).$
For any $1 \leq l \leq \min (n, |\mathcal{X}|/55)$ and any estimator $\hat P_n = \Phi (X^{(1)}, \dots, X^{(n)})$, where $\Phi$ is an estimator map $\Phi:[|\mathcal{X}|]^n\to \mathcal{P}(\mathcal{X})$, 
there exists a distribution $P \in \mathcal{P}_{l}$ such that
  \begin{equation}
    \label{eq:sparse-lower-high-probability}
    \mathbb{P}_P \bigg( \kll(P\;|\;\hat{P_n}) \geq \frac{l \log (e |\mathcal{X}|/l)}{300\, n} \bigg)
    \geq 1 - 3 \exp \Big( - \frac{l}{35} \Big) %
    \,.
  \end{equation}
\end{proposition}
\begin{corollary}[\citep{Mourtada_2026}, Section 4.1, Corollary 1]\label{cor:mourtada_lower_bound}
Let us assume $P\in \mathcal{P}_l(\mathcal{X}).$Let $n$ be the number of samples $(X^{(1)},\ldots,X^{(n)})$, and $n, |\mathcal{X}| \geq 110$, $l \in \{ 2, \dots, \min(n, |\mathcal{X}|/55) \}$ and $\delta \in (e^{-n}, e^{-2})$.
For any estimator confidence-dependent $\Phi = \Phi_{n,\delta} : [|\mathcal{X}|]^n \to \mathcal{P}(\mathcal{X})$, there exists a distribution $P \in \mathcal{P}_l(\mathcal{X})$ such that, letting $\hat P_n = \Phi (X_1, \dots, X_n)$, we have
  \begin{equation}\label{eq:minimax-sparse-lower}
    \mathbb{P}_P \bigg( \kll(P|\hat P_n) \geq \frac{l \log (e |\mathcal{X}|/l) + \log (|\mathcal{X}|) \log (1/\delta)}{320\, n} \bigg)
    \geq \delta
    \, .
\end{equation}
\end{corollary}

We list here other results that are going to be useful. In the following for $a\ge0$ and $b>0$ we set, with the convention $0\log 0:=0,$
\[
D(a,b):=a\log\frac ab-a+b,
\]
so that $ D(0,b)=b .$
Using $\phi$, one has $D(a,b)=\phi(a)-\phi(b)-\phi'(b)(a-b)$, i.e.\ $D$ is the Bregman divergence of $\phi$. In particular $D(a,b)=b\,h(a/b)$ with $h(u)=u\log u-u+1\ge0$, hence $D\ge0$ with equality iff $a=b$; additionally, $D$ is jointly convex and positively homogeneous, $D(ca,cb)=cD(a,b)$ for $c\ge0$.

\begin{lemma}[Relative entropy along two Markov flows]\label{lem:girsanov}
Let $\tau>0$ and let $u\mapsto A_u,B_u$ measurable maps from $[0,\tau]$ to $\mathbb R^{|\mathcal{X}|\times|\mathcal X|}.$ Called  $M_u:=\max_{x,y}\big(|A_u(x,y)|+|B_u(x,y)|\big)$, we assume $\int_0^\tau M_u\dd u<\infty$.
 Moreover, we take $A_u$ and $B_u$ as conservative generator in the column convention of Section~\ref{sec:preliminaries}.
 
 Now, let $(p_u)_{u\in[0,\tau]}$, $(q_u)_{u\in[0,\tau]}$ be the absolutely continuous solutions of $\partial_u p_u=A_up_u$, $\partial_u q_u=B_uq_u$ with $p_0,q_0\in\mathcal P(\mathcal X)$ and $q_0>0$ on $\mathcal X$, and assume that the support $S$ of $p_u$ does not depend on $u\in[0,\tau]$. Then, $q_u>0$ on $\mathcal X$ for every $u$, the map $u\mapsto\kll(p_u|q_u)$ is absolutely continuous on $[0,\tau]$, and, writing $r_u:=p_u/q_u$, for a.e.\ $u$
\begin{equation}\label{eq:kl_derivative}
\frac{d}{du}\kll(p_u|q_u)=\sum_{y\in\mathcal X}\sum_{x\neq y}\Big[A_u(x,y)\,p_u(y)\log\frac{r_u(x)}{r_u(y)}-B_u(x,y)\,q_u(y)\big(r_u(x)-r_u(y)\big)\Big],
\end{equation}
where above if a term satisfies $A_u(x,y)p_u(y)=0,$ then the whole term is substitute by $0$ (for a.e.\ $u$ this covers every pair with $x\notin S$ or $y\notin S$). In particular, with $D(a,0):=+\infty$ for $a>0$ and $D(0,0):=0$, for a.e.\ $u$
\begin{equation}\label{eq:kl_derivative_bound}
\frac{d}{du}\kll(p_u|q_u)\le\sum_{y\in\mathcal X}p_u(y)\sum_{x\neq y}D\big(A_u(x,y),B_u(x,y)\big).
\end{equation}
\end{lemma}

\begin{proof}

We start by showing the positivity of $q$. Since  solutions of Kolmogorov equations with bounded conservative generators stay in $\mathcal P(\mathcal X)$, we deduce immediately that $q_u\ge0$ and, by the column convention, $\partial_u q_u(x)\ge-(|\mathcal X|-1)M_u\,q_u(x)$; using Gr\"onwall's lemma for integrable coefficients, we get $q_u(x)\ge q_0(x)\exp\big(-(|\mathcal X|-1)\int_0^uM_z\,\mathrm d z\big)>0$.

We focus now on the absolute continuity. On the support $S=\supp(p_u)$ the functions $u\mapsto p_u(x)$ are continuous and positive, hence bounded away from $0$ on $[0,\tau]$, and similarly for $q_u(x)$. Thus 
\[
\kll(p_u|q_u)=\sum_{x\in S}[p_u(x)\log p_u(x)-p_u(x)\log q_u(x)]
\]
is a finite sum of products of absolutely continuous functions with smooth functions of absolutely continuous functions taking values in compact subsets of $(0,\infty)$; therefore, it is absolutely continuous and can be differentiated term by term for a.e.\ $u$.

For a.e.\ $u$ and every $x\notin S$, 
\[0=\partial_u p_u(x)=\sum_{y\neq x}A_u(x,y)p_u(y)
\]
is a sum of nonnegative terms, since $A,p_u\geq 0,$ thus
\begin{equation}\label{eq:K2}
A_u(x,y)\,p_u(y)=0
\end{equation}
whenever $x\notin S$ for a.e. $u$,
and $\sum_{x\in S}\partial_u p_u(x)=\partial_u\sum_xp_u(x)=0$. If we omit the index $u$ now, and using the two facts above and $r=0$ outside of $S$,
\begin{equation}\label{tag_1}
\frac{d}{du}\kll(p|q)=\sum_{x\in S}\partial_u p(x)\log r(x)-\sum_{x\in S}r(x)\partial_u q(x).
\end{equation}
Thanks to the the column convention, 
\[
\partial_u p(x)=\sum_{y\neq x}A(x,y)p(y)-p(x)\sum_{y\neq x}A(y,x).
\]
Hence, the first term in \eqref{tag_1} reads
\[
\sum_{x\in S}\partial p(x)\log r(x)=\sum_{x\in S}\sum_{y\neq x}A(x,y)p(y)\log r(x)-\sum_{x\in S}p(x)\sum_{y\neq x}A(y,x)\log r(x).
\]
In the second double sum only $y\in S$ contribute, by \eqref{eq:K2}.  Exchanging the names of $x$ and $y$ there, and extending the first double sum to all $x$ by \eqref{eq:K2} with the stated convention, we obtain $\sum_{y\in S}p(y)\sum_{x\neq y}A(x,y)[\log r(x)-\log r(y)]$. Similarly, using $r(x)=0$ for $x\notin S$ and $q(y)r(y)=p(y)$, so in particular
\begin{align*}
    \sum_{x\in S}r(x)\partial_u q(x)&=\sum_{y}\sum_{x\neq y}B(x,y)q(y)r(x)-\sum_{y}p(y)\sum_{x\neq y}B(x,y)\\
    &=\sum_y\sum_{x\neq y}B(x,y)q(y)\big(r(x)-r(y)\big).
\end{align*}
This proves \eqref{eq:kl_derivative}.

Finally, to prove the upper bound, we split the sum in \eqref{eq:kl_derivative} in $y\notin S$ or $y\in S$. If $y\notin S$ then $p(y)=r(y)=0$. Then, the term inside of the double sum in \eqref{eq:kl_derivative} is equal to 
\[
-\sum_{x\neq y}B(x,y)q(y)r(x)\le0.\]
On the other side, if $y\in S$, factor $q(y)=\frac{p(y)}{r(y)}$ and observe that $\frac{r(x)}{r(y)}\in[0,\infty)$: the term inside the double sum in \eqref{eq:kl_derivative} is now 
\[
p(y)\sum_{x\neq y}\Big[A\log \frac{r(x)}{r(y)}-B\frac{r(x)}{r(y)}+B\Big]
\]
where, to simplify notations, we wrote $A=A(x,y)$, $B=B(x,y)$. We show now that $$A\log \frac{r(x)}{r(y)}-B\frac{r(x)}{r(y)}+B\le A\log\frac AB-A+B=D(A,B).$$ In fact, if $A,B>0$ and $\frac{r(x)}{r(y)}>0$ the quantity above is $\log w\le w-1$ with $w=\frac{r(x)}{r(y)}\frac{B}{A}$, multiplied by $A$; on the other side, if $A=0$ both sides equal $B$; additionally, if $B=0<A$ the right side is $+\infty$; if $\frac{r(x)}{r(y)}=0<A$ the left side is $-\infty$ (a case excluded for a.e.\ $u$ by \eqref{eq:K2} anyway). To finish, we sum everything and obtain \eqref{eq:kl_derivative_bound}.
\end{proof}

To finish, we prove that for both kernels, the support of the distribution doesn't change in $(0,T].$ In particular, it is the whole space in the uniform case and the states $y$ such that there exists a state $x$ in the support of the initial law $\rho$ such that $y$ is obtained from $x$ masking some coordinates.
\begin{lemma}[Support of the forward marginals]\label{lem:support}
For both kernels, let us take $\rho\in\mathcal P(\mathcal X_0)$ as the initial distribution of the process $\rho_t=P_{0,t}\rho$. Then, 
\[
\supp(\rho_{t_1})=\supp(\rho_{t_2})\qquad \forall t_1,t_2 \in(0,T].
\]
In particular, $\supp(\rho_{t})=\mathcal X_0$ for all $t\in(0,T]$ for the uniform kernel, and $\supp(\rho_t)=\{y\in\mathcal X:\rho(X_{A(y)}=y_{A(y)})>0\}$ for the masked kernel and for every $t\in(0,T]$. Moreover, $\rho_t>0$ on $\mathcal X$ for every $t\in(0,T]$ if $\rho>0$ on $\mathcal X_0$, and finally on $[T_0,T]$, where $T_0>0$, the positive entries of $P_{0,t}$ are bounded below by a positive constant.
\end{lemma}
\begin{proof}
The statement follows immediately from the explicit form of the kernels: in fact in the uniform case, for all $y,x_0$,
$$P^{\mathrm{unif}}_{0,t}(y,x_0)\ge((1-\alpha_t)/K)^d>0;$$
on the other side, in the masked case $$P^{\mathrm{mask}}_{0,t}(y,x_0)=\alpha_t^{|A(y)|}(1-\alpha_t)^{d-|A(y)|}\mathbf 1\{x_{0,A(y)}=y_{A(y)}\}.$$
We observe that above $\alpha_t\in(0,1)$ for $t>0$ and $\alpha_t\le\alpha_{T_0}$, $1-\alpha_t\ge1-\alpha_{T_0}$ on $[T_0,T]$. This concludes.
\end{proof}

\section{Proofs of Section~\ref{sec:statitical_part}}\label{appendix_statistical_part}
\subsection{Proofs of Section~\ref{sec:entropy}}
\begin{proof}[Proof of Lemma~\ref{lem:proper_smoothing}]

For $B\subseteq[d]$ and $w\in\mathcal V^B$ let us call $n_B(w)$ the number of samples that are equal to $w$ in the positions $B$, i.e.
\[
n_B(w):=\#\{k\le n:\ X^{(k)}_{0,B}=w\}.
\] 
Thus, given a context $c=(A,v,i)$, we use the notation $n_c$ to denote $n_A(v)$. We notice that it only depends on $(A,v)$. We can then rewrite $$\psi^{(n)}_c(a)=\frac{n_{A\cup\{i\}}(v,a)}{n_A(v)}$$ when $n_c>0$. Since, by assumption, there exists $b_i\in\mathcal{V}$ such that $\data(X_i=b)=0, $ $n_{A\cup\{i\}}(v,b_i)=0$ for all $A$, $v$ and $i\notin A$. It follows that
\begin{equation}\label{36}
\hat\psi_c(b_i)=\frac{\varepsilon_c}{K}\qquad\text{for every context } c=(A,v,i),
\end{equation}
since $\psi^{(n)}_c(b_i)=0$ if $n_c>0$, while $\hat\psi_c=\Unif(\mathcal V)$ and $\varepsilon_c=1$ if $n_c=0$.

Let $\bar\mu\in\mathcal P(\mathcal X_0)$ with $\bar\mu>0$. Thanks to Proposition~2 applied to $\rho=\bar\mu$, for $t\in(0,T]$, $y\in\mathcal X$, $i\in M(y)$, $A=A(y)$ and $x=(y_{-i},x_i)$,
\[
s_{\bar\mu}(x,y,t)=\kappa_t\,\bar\mu(X_i=x_i\mid X_A=y_A).
\]
Notice that these are exactly the pairs $(x,y)$ with $Q_t(y,x)>0$. Hence $s=\kappa_t\hat\psi_{c(y,i)}(x_i)$ is proper if and only if there is $\bar\mu>0$ on $\mathcal X_0$ with
\begin{equation}\label{37}
\bar\mu(X_i=a\mid X_A=v)=\hat\psi_c(a)\qquad\text{for all } c=(A,v,i),\ a\in\mathcal V .
\end{equation}
Let us observe that such a $\bar\mu$ is unique, since, by the chain rule, $\bar\mu(x)=\prod_{j=1}^d\hat\psi_{([j-1],x_{<j},j)}(x_j)$.

We prove now the if part in the statement. Let $\gamma\in(0,\infty)$ and $\bar\mu:=\hat\mu_0$, i.e.
\[
\bar\mu(x)=\frac{n_{[d]}(x)+\gamma K^{-d}}{n+\gamma}>0.
\]
We will often omit $\gamma$ in the following, since it is fixed. Summing over the $K^{d-|A|}$ sequences $x$ with $x_A=v$ we get 
\[
\bar\mu(X_A=v)=\frac{n_A(v)+\gamma K^{-|A|}}{n+\gamma},
\]
and in a similar way 
\[
\bar\mu(X_A=v,X_i=a)=\frac{n_{A\cup\{i\}}(v,a)+\gamma K^{-|A|-1}}{n+\gamma}.
\]
Dividing,
\[
\bar\mu(X_i=a\mid X_A=v)=\frac{n_{A\cup\{i\}}(v,a)+\gamma K^{-|A|-1}}{n_A(v)+\gamma K^{-|A|}}=(1-\varepsilon_c)\,\psi^{(n)}_c(a)+\frac{\varepsilon_c}{K}
\]
with $\varepsilon_c$ as in \eqref{eq:eps} (note that when $n_c=0$ both sides equal $\frac1K$). Thus \eqref{37} holds, $s=s_{\bar\mu}$ is proper and $\mu_s=\hat\mu_0$. If $\gamma=\infty$, $\hat\psi_c=\Unif(\mathcal V)$ for every $c$, which are the conditionals of $\Unif(\mathcal X_0)$.

We work now on the only if part. To start, let $\bar\mu>0$ satisfy \eqref{37}. By \eqref{36}--\eqref{37} and positivity of $\bar\mu$,
\begin{equation}\label{38}
\varepsilon_c=K\,\frac{\bar\mu(X_A=v,X_i=b_i)}{\bar\mu(X_A=v)}>0\qquad\text{for every } c=(A,v,i).
\end{equation}
We take now $A\subseteq[d]$ with $|A|\le d-2$, $v\in\mathcal V^A$, $i\neq j$ in $[d]\setminus A$ and finally $a,a'\in\mathcal V$. Computing $\bar\mu(X_A=v,X_i=a,X_j=a')$ in the two possible orders and dividing by $\bar\mu(X_A=v)>0$,
\begin{equation}\label{39}
\hat\psi_{(A,v,i)}(a)\,\hat\psi_{(A\cup\{i\},(v,a),j)}(a')=\hat\psi_{(A,v,j)}(a')\,\hat\psi_{(A\cup\{j\},(v,a'),i)}(a).
\end{equation}

We prove now that $\varepsilon_c$ depends only on $(A,v)$. Taken $a=b_i$, $a'=b_j$ in \eqref{39}, the contexts $(A\cup\{i\},(v,b_i),j)$ and $(A\cup\{j\},(v,b_j),i)$ have zero count, so by assumption their $\hat\psi$ is $\Unif(\mathcal V)$;then,  the identity reads $\frac{\varepsilon_{(A,v,i)}}{K^{2}}=\frac{\varepsilon_{(A,v,j)}}{K^{2}}$. Hence $\varepsilon_{(A,v,i)}=\varepsilon_{(A,v,j)}$ whenever $|A|\le d-2$; moreover, if $|A|=d-1$ there is a single context with first two entries $(A,v)$, so they must coincide. We write from now on then  $\varepsilon_{A,v}$.

Take now $a=b_i$ and $a'$ arbitrary in \eqref{39}. By \eqref{36}, the uniqueness of $\varepsilon_c$ fixed $(A,v)$ and the fact that $(A\cup\{i\},(v,b_i),j)$ has zero count, after multiplying by $K$, we get in the case $|A|\le d-2,\ j\notin A,\ a'\in\mathcal V$, the other is trivial as shown above, 
\begin{equation}\label{40}
\frac{\varepsilon_{A,v}}{K}=\hat\psi_{(A,v,j)}(a')\,\varepsilon_{A\cup\{j\},(v,a')}.
\end{equation}
For $(B,w)$ with $n_B(w)>0$ and $|B|\le d-1$ define 
\[
g(B,w):=\frac{\varepsilon_{B,w}\,n_B(w)\,K^{|B|}}{(1-\varepsilon_{B,w})\in(0,\infty]},
\] 
with $g=\infty$ if and only if $\varepsilon_{B,w}=1$ ( we recall that $\varepsilon_{B,w}>0$ by \eqref{38}). Let now $|A|\le d-2, j\notin A$ and $a'$ with 
\[
m':=n_{A\cup\{j\}}(v,a')>0;
\]
then $n_A(v)\ge m'>0$ and, denoting for simplicity $\varepsilon$ the parameter $\varepsilon_{A,v}$, and  $\varepsilon'$ the parameter $\varepsilon_{A\cup\{j\},(v,a')}$, we have $\hat\psi_{(A,v,j)}(a')=\frac{(1-\varepsilon)m'}{n_A(v)}+\frac{\varepsilon}{K}$, so that \eqref{40} becomes 
\[
\frac{\varepsilon(1-\varepsilon')}{K}=\frac{(1-\varepsilon)\varepsilon' m'}{n_A(v)},
\]
i.e.
\begin{equation}\label{41}
\varepsilon\,n_A(v)\,K^{|A|}(1-\varepsilon')=(1-\varepsilon)\,\varepsilon'\,m'\,K^{|A|+1}.
\end{equation}
If both $\varepsilon$ and $\varepsilon'$ are less than $1$, \eqref{41} says $g(A,v)=g(A\cup\{j\},(v,a'))$. Furthermore, if $\varepsilon=1$ the left side vanishes, so $\varepsilon'=1$ since $m'>0$; finally, if $\varepsilon'=1$ the right side vanishes, so $\varepsilon=1$; in both cases $g(A,v)=g(A\cup\{j\},(v,a'))=\infty$. Hence $g(A,v)=g(A\cup\{j\},(v,a'))$ in every case.

We notice now that any $(B,w)$ with $n_B(w)>0$ and $|B|\le d-1$ is obtained from $(\emptyset,\emptyset)$ by adding the coordinates of $B$ one at a time. Additionally, every intermediate pair $(B',w_{B'})$ satisfies $n_{B'}(w_{B'})\ge n_B(w)>0$ and $|B'|\le d-2$ before the last step, hence the previous paragraph applies at each step and $g(B,w)=g(\emptyset,\emptyset)=:\gamma\in(0,\infty].$ Note moreover that $n_\emptyset(\emptyset)=n$. 

Solving $g(B,w)=\gamma$ for $\varepsilon_{B,w}$ gives $$\varepsilon_{B,w}=\frac{\gamma K^{-|B|}}{\gamma K^{-|B|}+n_B(w)},$$ which is \eqref{eq:eps} for all $c$ with $n_c>0$; in the case $n_c=0$ we have $\varepsilon_c=1$ by assumption, which is again \eqref{eq:eps}. To finish, by the if part above, $(\hat\psi_c)_c$ are the conditionals of $\hat\mu_0$, so $\bar\mu=\hat\mu_0$ by the uniqueness above.
\end{proof}
\begin{lemma}\label{lem:loss_contexts}
Consider the masked kernel. Let $s(x,y,t)=\kappa_t\psi_{c(y,i)}(x_i)$ with $x=(y_{-i},x_i),\psi_c\in\mathcal P(\mathcal V)$ for every context $c$, and for $k\in\{0,\dots,d-1\}$ set
\[
\omega_k:=\int_0^T\sigma_t\,\alpha_t^{k+1}(1-\alpha_t)^{d-k-1}\,dt\in(0,\infty).
\]
Then
\begin{align*}
\int_0^T\hat L_t(s)\,dt&=\sum_c\omega_{|A|}\frac{n_c}{n}\,H(\psi^{(n)}_c,\psi_c)+C_n,\\
\int_0^T L^U_t(s)\,dt&=\sum_c\omega_{|A|}K^{-|A|}\,H(\mathrm{Unif}(\mathcal V),\psi_c)+C_U,    
\end{align*}
where $C_n,C_U\in\mathbb R$ do not depend on $(\psi_c)_c$. It follows also then that, for every $\gamma>0$, with $\varepsilon_c$ as in \eqref{eq:eps},
\[
\int_0^T\hat L^M_t(s)\,dt=\sum_c\omega_{|A|}\,\frac{n_c+\gamma K^{-|A|}}{n+\gamma}\Big[(1-\varepsilon_c)H(\psi^{(n)}_c,\psi_c)+\varepsilon_c\,H(\mathrm{Unif}(\mathcal V),\psi_c)\Big]+\frac{nC_n+\gamma C_U}{n+\gamma},
\]
which is minimized, uniquely, at $\psi_c=\hat\psi_c=(1-\varepsilon_c)\psi^{(n)}_c+\varepsilon_c\mathrm{Unif}(\mathcal V)$ for every $c$. In a similar way, the first identity is minimized, uniquely, on the contexts with $n_c>0$, at $\psi_c=\psi^{(n)}_c$, while contexts with $n_c=0$ are unconstrained.
\end{lemma}

\begin{proof}
We fix $t\in(0,T]$. For the masked kernel, $Q_t(y,x)>0$ with $x\neq y$ if and only if $x=(y_{-i},a)$ for some $i\in M(y)$ and $a\in\mathcal V$: in this case, $Q_t(y,x)=\sigma_t$. Taking the average over the sample in \eqref{eq:loss_empirical} and using that $\frac1n\sum_kP_{0,t}(y,X^{(k)}_0)=\mu^{(n)}_t(y)$, $\frac1n\sum_kP_{0,t}(x,X^{(k)}_0)=\mu^{(n)}_t(x)$, we get
\begin{equation}\label{eq:L1}
\hat L_t(s)=\sigma_t\sum_{y\in\mathcal X}\sum_{i\in M(y)}\sum_{a\in\mathcal V}\Big[\mu^{(n)}_t(y)\,s(x,y,t)-\mu^{(n)}_t(x)\log s(x,y,t)\Big],\qquad x=(y_{-i},a).
\end{equation}
Fix $y$ with $M(y)\neq\emptyset$, $i\in M(y)$, and let $c=(A,v,i)$ with $A=A(y)$, $v=y_A$. By the product form of $P^{\mathrm{mask}}_{0,t}$, 
\[
\mu^{(n)}_t(y)=\alpha_t^{|A|}(1-\alpha_t)^{d-|A|}\,\frac{n_c}{n}
\]
and 
\[
\mu^{(n)}_t(x)=\alpha_t^{|A|+1}(1-\alpha_t)^{d-|A|-1}\,n_c\,\frac{\psi^{(n)}_c(a)}{n}.
\]
In both cases. they vanish if $n_c=0$. Since by definition $\kappa_t=\frac{\alpha_t}{1-\alpha_t}$ and $\sum_a\psi_c(a)=\sum_a\psi^{(n)}_c(a)=1$, the inner sum over $a$ in \eqref{eq:L1} equals
\[
\alpha_t^{|A|+1}(1-\alpha_t)^{d-|A|-1}\,\frac{n_c}{n}\Big[1-\log\kappa_t+H(\psi^{(n)}_c,\psi_c)\Big].
\]
As $(y,i)\mapsto c(y,i)$ is a bijection between $\{(y,i):y\in\mathcal X,\,i\in M(y)\}$ and the set of contexts,
\[
\hat L_t(s)=\sum_c\sigma_t\,\alpha_t^{|A|+1}(1-\alpha_t)^{d-|A|-1}\,\frac{n_c}{n}\Big[1-\log\kappa_t+H(\psi^{(n)}_c,\psi_c)\Big].
\]
Integrating in $t$ gives the first identity of the statement with 
\[
C_n:=\sum_c\frac{n_c}{n}\int_0^T\sigma_t\alpha_t^{|A|+1}(1-\alpha_t)^{d-|A|-1}(1-\log\kappa_t)\,dt.
\]
To check that $\omega_k$ and $C_n$ are finite, substitute $u=\bar\sigma_t$, with $\sigma_t\,dt=du$, $\alpha_t=e^{-u}$: 
\[
\omega_k=\int_0^{\bar\sigma_T}e^{-(k+1)u}(1-e^{-u})^{d-k-1}du\in(0,\bar\sigma_T],
\]
and since $1-e^{-u}\ge u/2$ for $u\le1$ and $1-e^{-u}\ge1/2$ for $u\ge1$ we have that $\log\kappa_t=-u-\log(1-e^{-u})$ with $|\log(1-e^{-u})|\le\log2+|\log u|$, so the integrand of $C_n$ is bounded by $e^{-(k+1)u}(1+u+\log2+|\log u|)$, which is integrable on $(0,\bar\sigma_T)$. 

Moreover, if we run the same computations as above with $\mu^{(n)}_0$ replaced by $\mathrm{Unif}(\mathcal X_0)$ (for which $\mathrm{Unif}(X_A=v)=K^{-|A|}$ and the conditional law of $X_i$ given $X_A=v$ is $\mathrm{Unif}(\mathcal V)$), the second identity follows. We observe that the term $\phi(s^\star_U)$ in \eqref{definition_score_entropy_loss} does not depend on $s$ and is absorbed in $C_U$. 

Finally, the identity for 
\[
\hat L^M_t=\frac{n}{n+\gamma}\hat L_t+\frac{\gamma}{n+\gamma}L^U_t
\]
follows since 
\[
\frac{n}{n+\gamma}\frac{n_c}{n}=\frac{n_c+\gamma K^{-|A|}}{n+\gamma}(1-\varepsilon_c)
\]
and 
\[
\frac{\gamma}{n+\gamma}K^{-|A|}=\frac{n_c+\gamma K^{-|A|}}{n+\gamma}\varepsilon_c.
\]
To finish, all weights $\frac{\omega_{|A|}(n_c+\gamma K^{-|A|})}{n+\gamma}$ are positive and the cross-entropy is linear in its first argument, hence 
\[
(1-\varepsilon)H(p,q)+\varepsilon H(u,q)=H((1-\varepsilon)p+\varepsilon u,q),
\]
and $H(p,q)=H(p,p)+\kll(p|q)\ge H(p,p)$ with equality if and only if $q=p$. This gives the two minimization statements.
\end{proof}

\subsection{Proofs of Section~\ref{sec:upper_bounds}}\label{appendix_proof_thm_minimax_mask}

To begin, we recall here that
a context $c$ is a triple $(A,v,i)$, where $A\subseteq [d],$ $v\in \mathcal{V}^{A}$ and $i\in [d]\setminus A.$ In particular, given a certain state $y\in \mathcal{X},$ we will often call $c(y,i)$, or for brevity when it isn't ambiguous $c$, the triple $(A(y), y_{A(y)},i),$ where $A(y)$ are the positions of unmasked tokens in $y$. We will denote also by $M(y):=[d]\setminus A(y),$ the position of the masked tokens in $y.$   
Given a position $i\in[d]$, a token $x_i\in \mathcal{V}$, and a state $y\in \mathcal{X}$ such that $y_i=\mathfrak{m}$, we can then define for a generic $\rho$ the following conditional probability $\psi_c(a):=\frac{\rho(X_A=v,X_i=a)}{\rho(X_A=v)},$ when $\rho(X_A=v)>0.$ 

We first introduce here a key result from \citet{ou2024radd}. It tells us that the score depends on time through an universal and explicit constant, and the data-dependence is independent from time and noise. 
\begin{proposition}[\citep{ou2024radd}, Theorem 1]\label{prop:ou}
Let $\rho\in\mathcal P(\mathcal X_0)$, and $\rho_t:=P^{\mathrm{mask}}_{0,t}\rho$. For every $t\in(0,T]$, every $y\in\mathcal X$ with $\rho_t(y)>0$ (equivalently $\rho(X_{A(y)}=y_{A(y)})>0$), every $i\in M(y)$ and $a\in\mathcal V$,
\[
\frac{\rho_t(x)}{\rho_t(y)}=\kappa_t\,\rho(X_i=a\mid X_{A(y)}=y_{A(y)}),\qquad x:=(y_{-i},a),\quad \kappa_t=\frac{\alpha_t}{1-\alpha_t}.
\]
\end{proposition}
\begin{proof}
The proof follows immediately from the structure of $P^{\mathrm{mask}}_{0,t}$: in fact, \[
\rho_t(y)=\alpha_t^{|A|}(1-\alpha_t)^{d-|A|}\rho(X_A=y_A)
\]
with $A=A(y)$, and 
\[
\rho_t(x)=\alpha_t^{|A|+1}(1-\alpha_t)^{d-|A|-1}\rho(X_A=y_A,X_i=a).
\]
Dividing, we conclude.
\end{proof}
Let us now define an important object. Given a score $s$ (measurable in $t$, positive where $Q_t(y,x)>0$) define the \emph{excess score entropy}
\begin{equation}\label{eq:excess}
\mathcal E_t(s):=\sum_{y\in\mathcal X}\mu_t(y)\sum_{x\ne y}Q_t(y,x)\,D\big(s^\star(x,y,t),\,s(x,y,t)\big)\in[0,\infty],
\end{equation}
with $D$ as above (and $D(a,0)=+\infty$ for $a>0$, $D(0,0)=0$). Expanding $D$ and using $\mu_t(y)s^\star(x,y,t)=\mu_t(x)$ whenever $Q_t(y,x)>0$, one checks that $\mathcal E_t(s)=L_t(s)$, with $L_t$ the score-entropy loss \eqref{definition_score_entropy_loss}; in particular $L_t(s^\star)=0$ and $L_t\ge0$.

Throughout the section, many results hold also for the uniform case. For this reason we give already now their lemmas and proofs.

\begin{lemma}[Bounds on the uniform score]\label{lem:A}
Consider the uniform kernel, $\gamma>0$, and $\hat s$ the score for the uniform case. Then, for all $t\in[0,T]$ and $x,y\in\cX_0$,
\[
M_\gamma^{-1}\le\hat s(x,y,t)\le M_\gamma,\qquad M_\gamma:=\frac{K^d(n+\gamma)}{\gamma},
\]
and $t\mapsto\hat s(x,y,t)$ is continuous. Moreover, for $0<T_0\le T$, 
\[
s^\star(x,y,t)\le \frac{K^d}{(1-\alpha_{T_0})^{d}}
\]
for all $t\in[T_0,T]$ and every $x,y\in\cX_0$.
\end{lemma}
\begin{proof}
Since we consider the uniform kernel, we know that $Q^{\mathrm{unif}}_0$ is symmetric, so $P_{0,t}=e^{\bar\sigma_tQ_0}$ is symmetric with unit column sums, or equivalently it is doubly stochastic. Hence 
\[
\hat\mu_t(x)=\sum_{x_0}P_{0,t}(x,x_0)\muhat(x_0)
\]
is a convex combination of the values of $\muhat$, thus 
\[
\frac{\gamma}{(K^d(n+\gamma))}=\min\muhat\le\hat\mu_t(x)\le\max\muhat\leq 1
\]
Now, dividing we get the two-sided bound and the continuity in $t$ is follows quickly. 

To end, for $s^\star$, we have that $\mu_t(x)\le1$ and 
\[
\mu_t(y)\ge\min_{x_0}P_{0,t}(y,x_0)\ge\Big(\frac{1-\alpha_t}{K}\Big)^d\ge\Big(\frac{1-\alpha_{T_0}}{K}\Big)^d
\]
on $[T_0,T]$.
\end{proof}

\begin{lemma}[Excess score entropy along the forward flow]\label{lem:B}
Consider the uniform kernel, $\gamma>0$, $0<T_0\le T, \hat\mu_t:=P_{0,t}\muhat$, and  $\hat s$ as in \eqref{eq:hat_s_uniform}. Then, the map $t\mapsto\KL(\mu_t\,|\,\hat\mu_t)$ is absolutely continuous on $[T_0,T]$, for a.e.\ $t\in[T_0,T]$ and it holds
\[
\frac{d}{dt}\KL(\mu_t\,|\,\hat\mu_t)=-\cE_t(\hat s)=-L_t(\hat s).
\]
Consequently,
\[
\int_{T_0}^T L_t(\hat s)\dd t=\KL(\mu_{T_0}|\hat\mu_{T_0})-\KL(\mu_T|\hat\mu_T)\le\KL(\mu_{T_0}|\hat\mu_{T_0})\le\KL(\mudata\,|\,\muhat).
\]
\end{lemma}
\begin{proof}
To begin, we apply Lemma~\ref{lem:girsanov} on $[T_0,T]$ with $A_t=B_t=Q_t=\sigma_tQ_0$, $p_t=\mu_t$, $q_t=\hat\mu_t$ and $\int\sigma<\infty$: we observe that $$q_{T_0}=\hat\mu_{T_0}>0$$ because $\muhat>0$ and Lemma~\ref{lem:support}. Moreover, again thanks to Lemma~\ref{lem:support}, the support of $\mu_t$ is $\cX_0$ for every $t>0$, hence it is constant. Calling for simplicity $r_t:=\frac{\mu_t}{\hat\mu_t}>0$, by \eqref{eq:kl_derivative}, for a.e.\ $t$, we have
\begin{align*}
\frac{d}{dt}\KL(\mu_t|\hat\mu_t)&=\sum_{y}\sum_{x\ne y}Q_t(x,y)\Big[\mu_t(y)\log\frac{r_t(x)}{r_t(y)}-\hat\mu_t(y)\big(r_t(x)-r_t(y)\big)\Big]\\
&=\sum_{y}\mu_t(y)\sum_{x\ne y}Q_t(x,y)\,\big(\log  \frac{r_t(x)}{r_t(y)}-\frac{r_t(x)}{r_t(y)}+1\big),
\end{align*}
where above we used $\hat\mu_t(y)=\mu_t(y)/r_t(y)$ and that $\frac{r_t(x)}{r_t(y)}=\frac{s^\star(x,y,t)}{\hat s(x,y,t)}\in(0,\infty)$. Exchanging the names of $x$ and $y$ in the double sum, since $\mu_t(x)=\mu_t(y)s^\star(x,y,t)=\mu_t(y)\,\hat s(x,y,t)\frac{r_t(x)}{r_t(y)}$, the summand becomes
\begin{align*}
    \mu_t(y)\,Q_t(y,x)\,\hat s(x,y,t)\,\frac{r_t(x)}{r_t(y)}\Big(-\log \frac{r_t(x)}{r_t(y)}-\frac{r_y(y)}{r_t(x)}+1\Big)&=-\mu_t(y)\,Q_t(y,x)\,\hat s(x,y,t)\,h(\frac{r_t(x)}{r_t(y)})\\
    &=-\mu_t(y)\,Q_t(y,x)\,D\big(s^\star(x,y,t),\hat s(x,y,t)\big), 
\end{align*}
where we recall that $h(u)=u\log u-u+1$ and $D(a,b)=b\,h(a/b)$. Summing gives $-\cE_t(\hat s)$.
Integrating on $[T_0,T]$, using that $\KL\ge0$, and finally the data processing inequality, we get the last inequality.
\end{proof}

\begin{lemma}[Transfer from score error to the sampler output]\label{lem:transfer}
Let $0<T_0<T, \mu_\infty\in\cP(\cX)$ with $\mu_\infty>0$ on $\cX$. Let $s$ be a score, measurable in $t$, with $\sup\{s(x,y,t):\ x,y\in\cX,\ t\in[T_0,T]\}<\infty$ and $s(x,y,t)>0$ whenever $s^\star(x,y,t)Q_t(y,x)>0$. Moreover, let $(\nu_u)_{u\in[0,T-T_0]}$ be the solution of \eqref{sampling_kolmogorov_equation} with score $s$ and initial law $\nu_0=\mu_\infty$. Then
\begin{equation}\label{44}
\KL(\mu_{T_0}\,|\,\nu_{T-T_0})\le\KL(\mu_T\,|\,\mu_\infty)+\int_{T_0}^T\cE_t(s)\dd t .
\end{equation}
In particular, for the uniform kernel, $\mu_\infty=\Unif(\cX_0)$ and $s=\hat s$ of \eqref{eq:hat_s_uniform}, the hypotheses hold by Lemma~\ref{lem:A} and \eqref{44} reads $\KL(\mu_{T_0}|\nu_{T-T_0})\le\KL(\mu_T|\Unif(\cX_0))+\int_{T_0}^TL_t(\hat s)\dd t$.
\end{lemma}
\begin{proof}
To start, we set $\tau:=T-T_0$ and, for $u\in[0,\tau]$ and $x\ne y$,
\[
\tilde Q_u(x,y):=s^\star(x,y,T-u)\,Q_{T-u}(y,x),\qquad \hat Q_u(x,y):=s(x,y,T-u)\,Q_{T-u}(y,x),
\]
with diagonals making the columns sum to zero, according to our convention. We observe that both are measurable in $u$. Moreover, it holds that $$\int_0^\tau\max_{x,y}(|\tilde Q_u|+|\hat Q_u|)\dd u<\infty,$$ in fact $|\hat Q_u(x,y)|\le\sup s\cdot\sigma_{T-u}\max|Q_0|$, and $s^\star_t$ is bounded on $[T_0,T]$: the latter follows in the masked case because $s^\star\le\kappa_t\le\kappa_{T_0}$ by Proposition~\ref{prop:ou}, and in the uniform case by Lemma~\ref{lem:A}, with $\int_0^T\sigma<\infty$.

Call now $p_u:=\mu_{T-u}$ and $q_u:=\nu_u$. By definition $q$ is the absolutely continuous solution of $\partial_u q_u=\hat Q_uq_u$, $q_0=\mu_\infty>0$. We claim that $\partial_u p_u=\tilde Q_up_u$ for a.e.\ $u$, with $p_0=\mu_T$. In order to prove it, for $x\in\cX$, we have that
\[
(\tilde Q_u\mu_{T-u})(x)=\sum_{y\ne x}s^\star(x,y,T-u)Q_{T-u}(y,x)\mu_{T-u}(y)-\mu_{T-u}(x)\sum_{y\ne x}s^\star(y,x,T-u)Q_{T-u}(x,y).
\]
Whenever $Q_t(y,x)>0$ one has $s^\star(x,y,t)\mu_t(y)=\mu_t(x)$: then, if $\mu_t(y)>0$ this is \eqref{definition_true_score}; otherwise, if $\mu_t(y)=0$ then $s^\star=0$ and also $\mu_t(x)=0$ Let us notice that for the uniform kernel $\mu_t>0$, so this case does not occur;on the other hand, in the masked case, $Q_t(y,x)>0$ means $y$ is obtained from $x$ by masking one coordinate, so $A(y)\subseteq A(x)$ and $\mu(X_{A(y)}=y_{A(y)})=0$ forces $\mu(X_{A(x)}=x_{A(x)})=0$). Applying this to both sums,
\begin{align*}
    (\tilde Q_u\mu_{T-u})(x)&=\mu_{T-u}(x)\sum_{y\ne x}Q_{T-u}(y,x)-\sum_{y\ne x}Q_{T-u}(x,y)\mu_{T-u}(y)\\
    &=-(Q_{T-u}\mu_{T-u})(x)\\
    &=\frac{d}{du}\mu_{T-u}(x),   
\end{align*}
by the column convention and (3). This is the time-reversal property behind \eqref{eq:ideal_chain}.

Now, the support of $p_u=\mu_{T-u}$ does not depend on $u\in[0,\tau]$ by Lemma~\ref{lem:support}, applied to $\rho=\mudata$ on $[T_0,T]$, and $q_0>0$. Moreover $\hat Q_u(x,y)>0$ whenever $\tilde Q_u(x,y)>0$ by the hypothesis on $s$, so $D(\tilde Q_u(x,y),\hat Q_u(x,y))<\infty$ for every pair. Integrating \eqref{eq:kl_derivative_bound} over $[0,\tau]$,
\[
\KL(\mu_{T_0}|\nu_\tau)=\KL(p_\tau|q_\tau)\le\KL(p_0|q_0)+\int_0^\tau\sum_y\mu_{T-u}(y)\sum_{x\ne y}D\big(\tilde Q_u(x,y),\hat Q_u(x,y)\big)\dd u .
\]
Since by homogeneity $D(ca,cb)=cD(a,b)$ for $c\ge0,$ with $D(0,0)=0$. Taking $c=Q_{T-u}(y,x)$ the inner sum equals $\sum_{x\ne y}Q_{T-u}(y,x)D(s^\star(x,y,T-u),s(x,y,T-u))$, and the substitution $t=T-u$ turns the integral into $\int_{T_0}^T\cE_t(s)\dd t$. This is \eqref{44}.
\end{proof}

\subsection{Proof of Theorem~\ref{thm:upper_bounds_mask}}

We want to prove in this subsection Theorem~\ref{thm:upper_bounds_mask}.In order to do that, we are going to prove some key lemmas before, and also we are going to need the following results.
\begin{lemma}[\citep{Mourtada_2026}, Section 6.1, Lemma 3]\label{lem:mourtada_gamma}
Let $\hat P_n=(\hat p_j)_j$ with $\hat p_j=\frac{N_j+\lambda}{n+\lambda K^d}$ for some $\lambda\in(0,n/K^d]$ that may depend on
$X_1,\dots,X_n$, and let $\bar p_j:=N_j/n$ be the empirical frequencies. Then
\[
\kll(P|\hat P_n)\le 6\sum_{j}\Big(\sqrt{\bar p_j}-\sqrt{p_j}\Big)^2+\frac{7\lambda K^d}{n}
+\sum_{j:\,p_j\ge4\lambda/n}p_j\log\Big(\frac{2np_j}{\lambda}\Big)\mathbf 1\Big(N_j\le\frac{np_j}{4}\Big).
\]
\end{lemma}
\begin{lemma}[\citep{Mourtada_2026},Lemma 4]\label{lem:mourtada_14}
We use the same notation as Lemma~\ref{lem:mourtada_gamma}. For $n,K^d\geq 2,$ and $\delta\in(0,1),$ with probability at least $1-2\delta$,
\begin{equation}
  \sum_{j\in\mathcal{X}_0 }\Big(\sqrt{\bar p_j} - \sqrt{p_j}\Big)^2<\frac{4 k_{\mathrm{eff}}(\data,n)+7 \log\frac1\delta}{n}.
\end{equation}
\end{lemma}
\begin{lemma}[\citep{Mourtada_2026},Lemma 5]\label{lem:mourtada_lemma_5}
In the same setting as Lemma~\ref{lem:mourtada_gamma}, for $n\geq 12, \lambda=1, K^d\geq 2$ and $\delta\in(e^{-n/6},e^{-2}),$ with probability at least $1-2\delta,$
\begin{equation}
R_\lambda:=\sum_{j:\,p_j\ge4\lambda/n}p_j\log\Big(\frac{2np_j}{\lambda}\Big)\mathbf 1\Big(N_j\le\frac{np_j}{4}\Big) \leq \frac{62000|\{x:p(x)\geq \frac4n\}|+106000\log\frac1\delta\log\log\frac1\delta}{n}.
\end{equation}
\end{lemma}
\begin{lemma}[\citep{Mourtada_2026}, Theorem 6]\label{lem:mourtada_theorem_6}
In the same setting as Lemma~\ref{lem:mourtada_gamma}, let
\begin{equation}
U_n:=\sum_x p(x)\mathbf{1}(N_x\leq np(x)/4).
\end{equation}
For $n,K^d\geq 2$ and $\delta\in(0,e^{-1}),$ with probability at least $1-8\delta$
\begin{equation}
    U_n\leq \frac{336 k_{\mathrm{eff}}^o(\data, n/112) + 2500e\log\frac1\delta}{n}.
\end{equation}
\end{lemma}

We define and recall now some objects that are going to be useful for the proof of the theorem. Given our sample $S=(X_0^{(1)},\dots,X_0^{(n)})$, we call $\mu_0^{(n)}:=\frac1n\sum_i\delta_{X_0^{(i)}}$ and $a_x:=n\data(x)$.  For
$\gamma\in[1,\min\{n,K^d\}]$ we set $L_\gamma:=\log(eK^d/\gamma)\ge1$ and denote by $\hat\mu_0^{(\gamma)}$ the estimator \eqref{eq:intro_smoothed_estimator} with parameter $\gamma$,
\[
\hat\mu_0^{(\gamma)}(x)=\frac{N_0(x)+\gamma/K^d}{n+\gamma},\qquad x\in\mathcal X_0 .
\]
The parameter $\gamma$ is a measurable map $\gamma:\mathcal X_0^n\to[1,\min\{n,K^d\}]$, and
$\hat\mu_0:=\hat\mu_0^{(\gamma(S))}$. A deterministic $\gamma$ is the special case of a constant map. Finally we define three quantities dependent on the sample, but not on the parameter
\begin{align*}
H_{n}&:=\sum_{x}\Big(\sqrt{\mu^{(n)}_0(x)}-\sqrt{\data(x)}\Big)^2,\qquad
U_n:=\sum_{x}\data(x)\,\mathbf 1\big(N_0(x)\le a_x/4\big),\\
V_n&:=\sum_{x:\,a_x\ge1/2}\data(x)\log(2a_x)\,\mathbf 1\big(N_0(x)\le a_x/4\big).
\end{align*}
They do not depend on $\gamma$ and are nonnegative, since $\log(2a_x)\ge0$ on the range of $V_n$. We observe that above $U_n$ is the mass of the classes that empirically are observed less that one forth of the their true frequency. 

We prove now Lemma~\ref{lem:upper_bound_expectation_mask}.

\begin{proof}[Proof of Lemma~\ref{lem:upper_bound_expectation_mask}]
To start, we notice that, thanks to Lemma~\ref{lem:mourtada_gamma} we can immediately bound our object. In order to do that, we fix a sample $S$ and a $\gamma\in[1,\min\{n,K^d\}].$ Taking $\lambda=\gamma/K^d$ and $\hat p=\hat\mu_0^{(\gamma)},$ we have that
\begin{equation}
    \kll(\data\mid \hat\mu_0^{(\gamma)})\leq 6H_n + \frac{7\gamma}{n}+ R_\gamma,
\end{equation}
where $$R_\gamma:=\sum_{x:a_x\geq 4\gamma/K^d}\data(x)\log\frac{2a_xK^d}{\gamma}\mathbf{1}(N_0(x)\leq a_x/4).$$
Moreover, we can rewrite the logarithm as $\log(2a_x)+\log\frac{K^d}{\gamma}.$ Then, $R_\gamma\leq \log\frac{K^d}{\gamma}U_n+V_n,$ where $\log\frac{K^d}{\gamma}\geq 0$ because $\gamma\leq K^d$ by assumption. We notice that above the range of the sums changed: in fact, in the case of $V_n$ it includes only the terms for which $\log(2a_x)\geq 0$, i.e. $a_x\geq 1/2,$ because the others are negative; in the other hand, the range in the case of $U_n$ is the whole $\mathcal{X}_0$, so we are adding more positive terms. We then arrive to the following upper bound
\begin{equation}\label{eq:jsaispas}
    \kll(\data\mid \hat\mu_0^{(\gamma)})\leq 6H_n+\frac{7\gamma}{n}+ V_n+\log\frac{K^d}{\gamma}U_n,
\end{equation}
where $H_n,V_n,U_n$ do not depend on $\gamma.$

We control now these three terms in expectation. To begin, we observe that when $\data(x)=0$ then $\bar p(x)=0,$ so those $x$ are negligible. For $x$ such that $\data(x)>0,$ we use that $(\sqrt{a}-\sqrt{b})^2=\frac{(a-b)^2}{(\sqrt{a}+\sqrt{b})^2}\leq \frac{(a-b)^2}{b}$ and that $(\sqrt{a}-\sqrt{b})^2\leq a+b.$ Moreover, $\mathbb{E}_S \mu_0^{(n)}(x)=\data(x)$ and, since our sample is i.i.d, $\mathbb{E}(\mu_0^{(n)}(x)-\data(x))^2=\data(x)(1-\data(x))/n,$ thus
\begin{equation}
    \mathbb{E}_S(6H_n)\leq 6\sum_x \min\{\frac1n, 2\data\}\leq \sum_x\frac{12}{n}\min\{1,a_x\}=\frac{12 k_{\mathrm{eff}}}{n}.
\end{equation}
Using again the fact that our sample is i.i.d, we observe immediately that $N_0(x)\sim \mathrm{Bin}(n,\data(x)).$ Thanks to the multiplicative Chernoff bound, $\mathbb{P}(N_0(x)\leq(1-\theta)n\data(x))\leq e^{-\theta^2 n\data(x)/2},$ so in particular $\mathbb{P}(N_0(x)\leq a_x/4)\leq e^{-9 n\data(x)/32},$ where we took $\theta=\frac34.$ Hence,
\begin{equation}
    \mathbb{E}_SU_n\leq \frac{1}{n}\sum_x a_xe^{-9a_x/32}\qquad \mathrm{and}\qquad \mathbb{E}_S V_n\leq \frac1n \sum_{x:a_x\geq 1/2}a_x\log(2a_x)e^{-9a_x/32}.
\end{equation}
We split the analysis of the terms inside the sums into two cases. If $a_x\leq 8,$ then $a_xe^{-9a_x/32}\leq a_x,$ since the exponent is negative, and $a_x\log(2a_x)e^{-9a_x/32}\leq a_x\log(16).$ On the other hand, if $a_x\geq 8,$ 
\begin{align*}
a_xe^{-9a_x/32}&=a_x e^{-5a_x/32} e^{-a_x/8}\\
&\leq e^{-a_x/8} \sup_{a_x>0} (a_x e^{-5a_x/32})\\
&= e^{-a_x/8} \frac{32}{5e}\\
&\leq e^{-a_x/8} e \\
&\leq e^{1-a_x/8}.
\end{align*}
Moreover, 
\begin{align*}
a_x\log(2a_x)e^{-9a_x/32}&\le\sqrt2\,a_x^{3/2}e^{-5a_x/32}e^{-a_x/8}\\
&\le\sqrt2\,e^{-a_x/8}\sup_{a>0}a^{3/2}e^{-5a/32}\\
&=\sqrt2\Big(\frac{48}{5e}\Big)^{3/2}e^{-a_x/8}\\
&\le 4e^{1-a_x/8},
\end{align*}
where above we used that $\log u\leq \sqrt{u}$ for $u>0$ and $\sup_{a>0} a^p e^{-ca}=(\frac{p}{ce})^p.$
If we call $m(a):=\min\{e^{1-a/8},a/8\},$ we notice that we gave as upper bound always a multiple of $m(a_x).$ Noticing that $\sum_x m(a_x)=k_{\mathrm{eff}}^o(\data,n/8),$ we get
\begin{equation}
    \mathbb{E}_S U_n\leq \frac{8 k^o_{\mathrm{eff}}(\data, n/8)}{n}\qquad \mathrm{and} \qquad \mathbb{E}_S V_n\leq \frac{23 k_{\mathrm{eff}}^o(\data,n/8)}{n}.
\end{equation}

We can now finish the proof. We just need to control the term $\log\frac{K^d}{\gamma}U_n$. We notice first of all that, since $\gamma$ depends on $S$ and the function $\gamma\mapsto \log\frac{K^d}{\gamma}$ is decreasing in $[\gamma_0,K^d]$, we have that if $\gamma \geq \gamma_0, \log\frac{K^d}{\gamma}\leq \log\frac{K^d}{\gamma_0}\leq\log\frac{eK^d}{\gamma_0}=L_{\gamma_0}.$ On the other hand, if $\gamma < \gamma_0$, then
\begin{align*}
\log\frac{K^d}{\gamma}&=\log\frac{K^d}{\gamma_0}+ \log\frac{\gamma_0}{\gamma}\\
&\leq \log\frac{K^d}{\gamma_0}+ \log{\gamma_0}\\
&\leq \log\frac{eK^d}{\gamma_0}+ \log{\gamma_0}\\
&=L_{\gamma_0}+ \log\gamma_0,
\end{align*}
where we used that $\gamma\geq 1.$ Then, to write both cases in a compact notation, 
\begin{equation}
    \log\frac{K^d}{\gamma}\leq L_{\gamma_0}+\log\gamma_0\mathbf{1}(\gamma<\gamma_0).
\end{equation}
Then, using $0\leq U_n\leq 1$
\begin{equation}
    \mathbb{E}_S[\log\frac{K^d}{\gamma}U_n]\leq L_{\gamma_0}\mathbb{E}_SU_n + \log\gamma_0\mathbb{P}(\gamma<\gamma_0).
\end{equation}
Putting all the results above together, we conclude.
\end{proof}
\begin{proof}[Proof of Lemma~\ref{lem:upper_bound_probability_mask}]
To start, we recall the setting of the lemma above. Thanks to Lemma~\ref{lem:mourtada_gamma} we can immediately bound our object: fixing a sample $S$, a $\gamma\in[1,\min\{n,K^d\}],\lambda=\gamma/K^d,$ and $\hat p=\hat\mu_0^{(\gamma)},$ we have that
\begin{equation}
    \kll(\data\mid \hat\mu_0^{(\gamma)})\leq 6H_n + \frac{7\gamma}{n}+ R_\gamma,
\end{equation}
where $$R_\gamma:=\sum_{x:a_x\geq 4\gamma/K^d}\data(x)\log\frac{2a_xK^d}{\gamma}\mathbf{1}(N_0(x)\leq a_x/4).$$
Moreover, rewriting the logarithm, $R_\gamma\leq \log\frac{K^d}{\gamma}U_n+V_n,$ where $\log\frac{K^d}{\gamma}\geq 0$ because $\gamma\leq K^d.$ We then arrive to 
\begin{equation*}
    \kll(\data\mid \hat\mu_0^{(\gamma)})\leq 6H_n+\frac{7\gamma}{n}+ V_n+\log\frac{K^d}{\gamma}U_n,
\end{equation*}
where $H_n,V_n,U_n$ do not depend on $\gamma.$

Let now $A_1$, $A_2$, $A_3$ be the events of Lemmas~\ref{lem:mourtada_14}, \ref{lem:mourtada_lemma_5}
and \ref{lem:mourtada_theorem_6}( notice that the last one applies since $\delta<e^{-2}<e^{-1}$), of probabilities at least $1-2\delta$, $1-2\delta$,
$1-8\delta$ respectively; each is defined through $S$ only. On $A_1$, 
\begin{equation}
6H_{n}\le\frac{24k_{\mathrm{eff}}+42\log\frac1\delta}{n}.
\end{equation}
For $V_n$, the terms with
$\frac12\le a_x<4$ satisfy $\log(2a_x)<\log8$ and contribute at most $$\log8\sum_{x:\,\data(x)<4/n}\data(x).$$ 
By the second expression
of $k_{\mathrm{eff}}$ in Definition~\ref{def:effective_support_missing},
$$\sum_{x:\,\data(x)<4/n}\data(x)\le\frac1n\big[\sum_{x:\,\data(x)<1/n}a_x+4\,|\{x:\data(x)\ge\frac1n\}|\big]\le\frac{4k_{\mathrm{eff}}}{n};$$
the terms with $a_x\ge4$ are exactly those of $R_{K^d}$. Hence $V_n\le\frac{4\log8\,k_{\mathrm{eff}}}{n}+R_{K^d}$ and, on $A_2$, using
$|\{x:\data(x)\ge4/n\}|\le k_{\mathrm{eff}}$ and $4\log8\le9$,
\begin{equation}
V_n\le\frac{62009\,k_{\mathrm{eff}}+106000\log\frac1\delta\log\log\frac1\delta}{n}.
\end{equation}

On $A_3$,
\begin{equation}
 U_n\le\frac{336 k_{\mathrm{eff}}^o(\data, n/112)+2500e\log\frac1\delta}{n}   
\end{equation}
 with $2500e\le6796$. 
 
To finish, let us set $A:=A_1\cap A_2\cap A_3$, so
$\mathbb P(A)\ge1-12\delta$ by the union bound. On $A$, for every $\gamma\in[1,\min\{n,K^d\}]$, insert the three bounds and
$\log\frac{K^d}{\gamma}\le L_\gamma$ into \eqref{eq:jsaispas}:
\begin{align*}
\kll\big(\data|\hat\mu_0^{(\gamma)}\big)&\le\frac{(24+62009)k_{\mathrm{eff}}+7\gamma+336\,k_{\mathrm{eff}}^o(\data, n/112)L_\gamma}{n}\\
&\quad+\frac{(42+6796L_\gamma)\log\tfrac1\delta+106000\log\tfrac1\delta\log\log\tfrac1\delta}{n},
\end{align*}
which is the claim. The bound holds for all $\gamma$ on the same event $A$, and $A$ does not depend on the
parameter; therefore, for any parameter $\gamma$, on $A$ it holds.
\end{proof}
\begin{proof}[Proof of Lemma~\ref{lem:sampler_mask}]

To simplify the notations, for $y\in\cX$ we write $\hat\mu_0(y_{A(y)}):=\muhat(X_{A(y)}=y_{A(y)})$, with $\hat\mu_0(y_\emptyset):=1$; since $\muhat>0$ on $\cX_0$, $\hat\mu_0(y_{A(y)})>0$ for every $y$. Using the product form of $P^{\mathrm{mask}}_{0,t}$,
\[
\hat\mu_t(y)=\alpha_t^{|A(y)|}(1-\alpha_t)^{d-|A(y)|}\,\hat\mu_0(y_{A(y)})>0,\qquad t\in(0,T].
\]
Thus, we can rewrite the estimator of \eqref{eq:conditional_is_smoothed} for $x,y$ such that $Q_t(y,x)>0$ as
\begin{equation}\label{eq:def_hat_s_context}
\hat{s}(x,y,t)=\frac{\alpha_t}{1-\alpha_t}\frac{\hat\mu_0(x_{A(x)})}{\hat\mu_0(y_{A(y)})}\qquad \text{with} \qquad x=(y_{-i},x_i),i\in M(y).
\end{equation}

We observe now that $\hat{Q}_u$ explodes when $u \to T$ since $\kappa_{T-u}\to \infty$ and that is why we take the limit and we can not work on $[0,T]$. Also, recall that, for $x\neq y$, 
$$\hat Q_u(y,x)=\hat s(y,x,T-u)\,Q_{T-u}(x,y),$$ 
where $Q_{T-u}(x,y)$ is the forward rate from $y$ to $x$. Since the forward process only masks tokens, $Q_{T-u}(x,y)\neq 0$ if and only if $x$ is obtained from $y$ by masking one unmasked coordinate, that is, if and only if $y=(x_{-i},y_i)$ for some $i\in M(x)$ and $y_i\in\mathcal V$. In that case, $Q_{T-u}(x,y)=\sigma_{T-u}$. Therefore, for such pairs
\begin{equation}\label{eq:reverse_rates_mask}
\hat Q_u(y,x)=\frac{\sigma_{T-u}\,\alpha_{T-u}}{1-\alpha_{T-u}}\;\frac{\hat\mu_0(y_{A(y)})}{\hat\mu_0(x_{A(x)})},
\end{equation}
and $\hat Q_u(y,x)=0$ for every other pair $x\neq y$.
For simplicity, we are going to denote from now on $\lambda_u:=\frac{\sigma_{T-u}\alpha_{T-u}}{1-\alpha_{T-u}}.$ Then, equation \eqref{sampling_kolmogorov_equation} reads
\begin{equation}\label{eq:linear_system_masked}
\partial_u\nu_u(y)=\sum_{x\neq y} \hat{Q}_{u}(y,x)\nu_u(x)-\nu_u(y)\sum_{x \neq y}\hat{Q}_u(x,y), \qquad \mathrm{with}\qquad \nu_0=\delta_{\mathfrak{m}^d}.    
\end{equation}
Since for every $T_0$ such that $0<T_0<T$ and $u\in[0, T-T_0]$ we have $0\leq \lambda_u\leq \frac{\sigma_{T-u}}{1-\alpha_{T_0}}$, that is integrable in $u$, we have that \eqref{eq:linear_system_masked} is a linear system with locally integrable coefficients in $[0,T).$ By the Carathéodory existence and uniqueness theorem, since the system is linear with coefficients in $L^1_{\mathrm{loc}}([0,T)),$   it has a unique absolutely continuous solution $(\nu_u)_{u\in[0,T)}.$ We want to write this solution explicitly.  

Since $u\in[0,T),$ the map $t\mapsto \log(1-\alpha_t)$ is absolutely continuous on $[T-u,T]$ with derivative $\frac{\sigma_t\alpha_t}{1-\alpha_t}$ a.e., so
\begin{equation}
    \Lambda_u:=\int_0^u\lambda_z dz = \int_{T-u}^T \frac{\sigma_t\alpha_t}{1-\alpha_t} dt= \log \frac{1-\alpha_T}{1-\alpha_{T-u}}.
\end{equation}
Moreover, we define
\begin{equation}
    F_u:=1-e^{-\Lambda_u}=\frac{\alpha_{T-u}-\alpha_T}{1-\alpha_T}\in [0,1).
\end{equation}
We focus on $F.$ It is an absolutely continuous function on $[0,T)$ with $F_0=0,F_u\to 1$ as $u\to T$ due to the fact that $\alpha_{T-u}\to \alpha_0=1,$ and $\frac{d}{du}F_u=\lambda_u(1-F_u).$ Thanks to $F_u,$ we can define 
\begin{equation}\label{eq:last_identity}
    \tilde{\nu}_u(y):=F_u^{|A(y)|}(1-F_u)^{d-|A(y)|}\hat\mu_0(y_{A(y)}),\qquad y\in\mathcal{X},u\in[0,T).
\end{equation}
We assume the convention $0^0=1$. 

We claim that $\tilde{\nu}_u$ is the explicit solution of \eqref{eq:linear_system_masked}. First, we notice that the initial condition is $\delta_{\mathfrak{m}^d}$ since $F_0=0$ makes $\tilde{\nu}_0(y)=0$ for all $y$ such that $A(y)\neq \emptyset.$ Fix now $y,k=|A(y)|$ and $A:=A(y).$ By \eqref{eq:linear_system_masked}, we can write 
$$\sum_{x \neq y}\hat{Q}_u(x,y)=\lambda_u\sum_{i\in M(y)}\sum_a \frac{\hat\mu_0(x_{A\cup i})}{\hat\mu_0(y_A)}=\lambda_u(d-k).$$
Since the state $x$ such that $\hat{Q}_u(y,x)\neq 0$ are exactly the ones that differs from $y$ by only one token in a position $i\in A$, with a mask token in the case of $x,$ thanks to \eqref{eq:def_hat_s_context} and the identity in \eqref{eq:last_identity}, we have
\begin{equation}
    \hat{Q}_u(y,x)\tilde{\nu}_u(x)=\lambda_u\frac{\hat\mu_0(y_{A})}{\hat\mu_0(y_{A\setminus i})} F_u^{k-1}(1-F_u)^{d-k+1}\hat{\mu}_0(y_{A\setminus i})=\lambda_uF_u^{k-1}(1-F_u)^{d-k+1}\hat\mu_0(y_A).
\end{equation}
If we sum over $i\in A,$ the right hand side of \eqref{eq:linear_system_masked}, at $\tilde{\nu}_u$ is equal to 
$$\lambda_u\hat\mu_0(y_A)[kF_u^{k-1}(1-F_u)^{d-k+1}-(d-k)F_u^k(1-F_u)^{d-k}].$$
Furthermore, for a.e. $u$ we also have that
\begin{align*}
    \frac{d}{du}[F_u^k(1-F_u)^{d-k}]&=[kF_u^{k-1}(1-F_u)^{d-k}- (d-k)F_u^k(1-F_u)^{d-k-1}]\frac{d}{du}F_u\\
    &=\lambda_u[kF_u^{k-1}(1-F_u)^{d-k+1}-(d-k)F_u^{k}(1-F_u)^{d-k}].
\end{align*}
Above the terms with coefficient $k=0$ or $d=k$ are absent. We can conclude then that the explicit solution for every $u\in[0,T)$ is $\tilde{\nu}_u.$

To finish, for $x\in \mathcal{X}_0,$ we have that $\nu_u(x)=F_u^d\hat\mu_0(x)\to \hat\mu_0(x),$ by continuity, and for $y\notin \mathcal{X}_0, 0\leq \nu_u(y)\leq (1-F_u)\to 0$ as $u$ goes to $T.$ Thus, defining $\hat\nu:=\lim_{u\to T}\nu_u,$ we have that the limit exists and it is equal to $\hat\mu_0.$
\end{proof}

We are now ready for the proof of Theorem~\ref{thm:upper_bounds_mask}.

\begin{proof}[Proof of Theorem~\ref{thm:upper_bounds_mask}]
Fix a sample $S$. Then $\gamma=\gamma(S)\in[1,\min\{n,K^d\}]$ is a number, $\hat\mu_0^{(\gamma(S))}$ is a strictly positive law on $\cX_0$, and $\hat s$ of \eqref{eq:conditional_is_smoothed} is its score; Lemma~\ref{lem:sampler_mask} gives
$$\KL(\mudata|\hat\nu)=\KL(\mudata|\hat\mu_0^{(\gamma(S))})$$
for this $S$. Then, the first item in the theorem's statement follows by taking expectations and applying Lemma~\ref{lem:upper_bound_expectation_mask}. Finally, the second item follows by intersecting with the event $A$ of Lemma~\ref{lem:upper_bound_probability_mask}, which does not depend on $\gamma$.

\end{proof}
\subsection{Proofs of Section~\ref{sec:sparse_case}}\label{appendix_final}

Throughout this subsection we write $k_{\mathrm{eff}}:=k_{\mathrm{eff}}(\data,n)$, $L_\gamma:=\log\frac{eK^d}{\gamma}$ for $\gamma\in[1,\min\{n,K^d\}]$ and
$\Lambda_l:=L_l=\log\frac{eK^d}{l}$. We use repeatedly the elementary facts
\begin{equation}\label{eq:elementary}
k_{\mathrm{eff}}^o(\data,m)\ \le\ k_{\mathrm{eff}}(\data,m)\ \le\ k_{\mathrm{eff}}\ \le\ \min\{|\mathrm{supp}\,\data|,\,K^d\}\quad(m\le n),
\end{equation}
because $\min\{e^{1-a},a\}\le\min\{1,a\}\le\min\{1,a'\}$ for $0\le a\le a'$, and $u\mapsto u\log\tfrac{eK^d}{u}$ is nondecreasing on $(0,K^d]$ because the
derivative is $\log(K^d/u)\ge0$.

We start now with an important lemma.

\begin{lemma}[Distinct observed states]\label{lem:D_n}
Let $s'_n:=\mathbb E_S[D_n]$. Then:
\begin{enumerate}
\item[(a)] $(1-e^{-1})\,k_{\mathrm{eff}}\ \le\ s'_n\ \le\ k_{\mathrm{eff}}$;
\item[(b)] for every $\delta\in(0,1)$, $\ \mathbb P_S\big(D_n\ge2k_{\mathrm{eff}}+2\log\frac1\delta\big)\le\delta$;
\item[(c)] $\ \mathbb P_S\big(D_n\le0.3\,k_{\mathrm{eff}}\big)\le e^{-k_{\mathrm{eff}}/21}$.
\end{enumerate}
\end{lemma}
\begin{proof}
(a) We begin noticing that $D_n=\sum_x\mathbf 1\{N_0(x)>0\}$ and $N_0(x)\sim\mathrm{Bin}(n,\data(x))$, so we have
$$s'_n=\sum_x[1-(1-\data(x))^n].$$ 
For $p\in[0,1]$, Bernoulli's inequality $(1-p)^n\ge1-np$ gives
$1-(1-p)^n\le\min\{1,np\}$; conversely $1-(1-p)^n\ge1-e^{-np}\ge(1-e^{-1})\min\{1,np\}$, because
$u\mapsto1-e^{-u}$ is nondecreasing and concave with value $1-e^{-1}$ at $u=1$. Summing over $x$ and recalling
$k_{\mathrm{eff}}=\sum_x\min\{1,n\data(x)\},$ we obtain (a).

(b)--(c) We notice now that the indicators $\mathbf 1\{N_0(x)>0\}$, $x\in\cX_0$, are negatively associated, see \citep{Dubhashi_Ranjan_1996}. Thus, $D_n=\sum_x\mathbf 1\{N_0(x)>0\}$ satisfies the Chernoff bounds of a sum of independent Bernoulli variables with the same means, see \citep{ben_hamou_lemma_9} for more details: for every $s>s'_n$ and every $s<s'_n$ respectively,
\[
\Prob(D_n\ge s)\le\exp\big(-D(s,s'_n)\big),\quad\Prob(D_n\le s)\le\exp\big(-D(s,s'_n)\big),\quad D(u,v)=u\log\frac uv-u+v .
\]
Let denote now $L:=\log(\frac1\delta)$ and $s:=2(s'_n+L)$. With $u:=\frac{s}{s'_n}=2+2\ell$, where $\ell:=L/s'_n$, we have $D(s,s'_n)=s'_n\,((2+2\ell)\log(2+2\ell)-(1+2\ell))$. Now, when $\ell=0,$ the quantity inside the parentheses is equal to $2\log2-1>0$ and for a generic $\ell$ its derivative is $2\log(2+2\ell)\ge2\log2>1$, so $(2+2\ell)\log(2+2\ell)-(1+2\ell)\ge\ell$ and $D(s,s'_n)\ge L$. Hence $$\Prob(D_n\ge2(s'_n+L))\le\delta,$$ and since $s'_n\le\keff$ by (a), $$\{D_n\ge2\keff+2L\}\subseteq\{D_n\ge2(s'_n+L)\}.$$ This proves (b).

Take now $s:=\frac{s'_n}{2}$: $$D(s'_{\frac{n}{2}},s'_n)=\frac{1-\log2}{2}s'_n\ge\frac{(1-\log2)(1-e^{-1})}{2}\keff\ge\keff/21$$ by (a), and $s'_{\frac{n}{2}}\ge\frac{1-e^{-1}}{2}\keff\ge0.3\keff,$ so 
$$\{D_n\le0.3\keff\}\subseteq\{D_n\le s'_n/2\}.$$ 
This concludes.

\end{proof}
We recall now the following result.
\begin{proposition}[\citep{Mourtada_2026}, Proposition~2]\label{prop:mourtada_prop2}
Let $n,K^d\ge2$. Then
$$\ \mathbb E_S\,\kll\big(\data\,|\,\hat\mu_0^{(D_n)}\big)\le
\frac{2.4\,k_{\mathrm{eff}}+2\,k_{\mathrm{eff}}^o(\data, n/2)\log\big(eK^d/k_{\mathrm{eff}}^o(\data, n/2)\big)}{n+1}.$$
\end{proposition}
Thanks to the results above, we can now prove that if we take $\gamma=D_n,$ we have an adaptive estimator that reaches the minimax rates. We prove it in the following corollary.
\begin{corollary}[Adaptive smoothing]\label{cor:adaptive_gamma}
Let us take $\gamma=D_n,$ $\hat\mu_0=\hat\mu_0^{(D_n)}$ and let us call for simplicity $k:=k_{\mathrm{eff}}$. Then, the following holds.
\begin{enumerate}[leftmargin=2em]
\item[(i)] If $n,K^d\ge2$,
\[
\mathbb E_S\,\kll(\data\,|\,\hat\mu_0)\ \le\ \frac{2.4\,k+2\,k_{\mathrm{eff}}^o(\data, n/2)\log\big(eK^d/k_{\mathrm{eff}}^o(\data, \frac{n}{2})\big)}{n+1}.
\]
\item[(ii)] If $n\ge13$, $K^d\ge3$ and $\delta\in(e^{-n/6},e^{-2})$, then with probability at least $1-14\delta$,
\[
\kll(\data\,|\,\hat\mu_0)\ \le\ 130000\,\frac{k+k_{\mathrm{eff}}^o(\data,\frac{n}{112})\log\frac{eK^d}{k}+\log(K^d)\log\frac1\delta+\log\frac1\delta\log\log\frac1\delta}{n}.
\]
\end{enumerate}
In particular, for $\data\in\mathcal P_l(\mathcal X_0)$ with $2\le l\le\min\{n,\lfloor K^d/55\rfloor\}$ and $\delta\ge e^{-K^d}$, the right-hand sides are at most
\[
C\,\frac{l\log\frac{eK^d}{l}}{n}\qquad\text{and}\qquad C\,\frac{\Gamma_\delta}{n},
\]
simultaneously for every $l$ as above. Thus, under the assumptions of Corollary~\ref{prop:minimax_mask_prob}, the upper bounds in its
two items hold for $\hat\nu$ driven by $\hat s=\hat s^{(D_n)}$, with the same estimator and the same event for every
$l$, and $\hat\nu$ is minimax optimal over $\mathcal P_l(\mathcal X_0)$ in both senses.
\end{corollary}
\begin{proof}[Proof of Corollary~\ref{cor:adaptive_gamma}]
(i) follows immediately from Proposition~\ref{prop:mourtada_prop2}.

(ii) Let $A$ be the event of Lemma~\ref{lem:upper_bound_probability_mask}, $\mathbb P_S(A)\ge1-12\delta$, and
$B:=\{D_n<2k_{\mathrm{eff}}+2\log\frac1\delta\}$, $\mathbb P_S(B)\ge1-\delta$ by Lemma~\ref{lem:D_n}(b). Since
$1\le D_n\le\min\{n,K^d\}$, the parameter $\gamma=D_n$ is admissible in Lemma~\ref{lem:upper_bound_probability_mask}. On the event
$A\cap B$, its explicit form with $\gamma=D_n$ and $7D_n\le14k_{\mathrm{eff}}+14\log\frac1\delta$
give
\begin{align}\label{eq:cor_start}
\kll(\data|\hat\mu_0)\le
\frac{63114\,k_{\mathrm{eff}}+56\log\frac1\delta+\big(336\,k_{\mathrm{eff}}^o(\data, n/112)+6796\log\tfrac1\delta\big)L_{D_n}}{n}\\
+\frac{106000\log\tfrac1\delta\log\log\tfrac1\delta}{n}.\nonumber
\end{align}
We still need to bound $L_{D_n}=\log(eK^d/D_n)$. Since $K^d\ge3$, $\log K^d\ge1$ and $\log(eK^d)=1+\log K^d\le2\log K^d$;
then, since $D_n\ge1$, $L_{D_n}\le\log(eK^d)\le2\log K^d$ always.

We study now the case $k_{\mathrm{eff}}\ge21\log\frac1\delta$. Let $E:=\{D_n>0.3\,k_{\mathrm{eff}}\}$. Thanks to
Lemma~\ref{lem:D_n}(c), we know that $\mathbb P_S(E^c)\le e^{-k_{\mathrm{eff}}/21}\le\delta$. On $E$,
\[
L_{D_n}\le\log\frac{eK^d}{0.3\,k_{\mathrm{eff}}}=\log\frac{eK^d}{k_{\mathrm{eff}}}+\log\frac{10}{3}
\le\log\frac{eK^d}{k_{\mathrm{eff}}}+1.21\le2.21\log\frac{eK^d}{k_{\mathrm{eff}}},
\]
because $k_{\mathrm{eff}}\le K^d$ gives $\log\frac{eK^d}{k_{\mathrm{eff}}}\ge1$. So, on $A\cap B\cap E$, of
probability at least $1-14\delta$,
$$336\,k_{\mathrm{eff}}^o(\data, n/112)L_{D_n}\le743\,k_{\mathrm{eff}}^o(\data, n/112)\log\frac{eK^d}{k_{\mathrm{eff}}}$$ 
and
$$6796\log\frac1\delta\,L_{D_n}\le13592\log(K^d)\log\frac1\delta.$$

On the other hand, in the case $k_{\mathrm{eff}}<21\log\frac1\delta$,  on $A\cap B$, of probability at least $1-13\delta$, using $k_{\mathrm{eff}}^o(\data, n/112)\le k_{\mathrm{eff}}<21\log\frac1\delta$ and $L_{D_n}\le2\log K^d,$ we have
\[
336\,k_{\mathrm{eff}}^o(\data, n/112)L_{D_n}\le336\cdot21\cdot2\,\log(K^d)\log\frac1\delta=14112\log(K^d)\log\frac1\delta
\]
and
$$6796\log\frac1\delta\,L_{D_n}\le13592\log(K^d)\log\frac1\delta.$$

In both cases, inserting into \eqref{eq:cor_start} and using $56\log\frac1\delta\le56\log(K^d)\log\frac1\delta$,
with probability at least $1-14\delta$,
\begin{align*}
\kll(\data|\hat\mu_0) &\le
\frac{63114\,k_{\mathrm{eff}}+743\,k_{\mathrm{eff}}^o(\data, n/112)\log\frac{eK^d}{k_{\mathrm{eff}}}+27760\log(K^d)\log\frac1\delta}{n}\\
& \qquad +\frac{106000\log\frac1\delta\log\log\frac1\delta}{n},
\end{align*}
which implies the claim.

We focus now on the sparse case. Let $\data\in\mathcal P_l(\mathcal X_0)$ with $2\le l\le\min\{n,\lfloor K^d/55\rfloor\}$ and
$\delta\ge e^{-K^d}$. We note that $\Lambda_l\ge1$, and we are going to use
$k_{\mathrm{eff}}^o(\data,m)\le k\le l$, the monotonicity of $u\mapsto u\log(eK^d/u)$ on $[1,\min\{n,K^d\}],$ and that
$\log\log\frac1\delta\le\log(K^d)$.

By \eqref{eq:elementary}, $k_{\mathrm{eff}}^o(\data, m)\le k_{\mathrm{eff}}\le l\le K^d$ for
$m\le n$, hence $k_{\mathrm{eff}}\le l\Lambda_l$ and, thanks to monotonicity,
\[
k_{\mathrm{eff}}^o(\data, n/2)\log\frac{eK^d}{k_{\mathrm{eff}}^o(\data, n/2)}\le l\Lambda_l
\]
and
\[
k_{\mathrm{eff}}^o(\data, n/112)\log\frac{eK^d}{k_{\mathrm{eff}}}\le k_{\mathrm{eff}}\log\frac{eK^d}{k_{\mathrm{eff}}}\le l\Lambda_l.
\]

Therefore (i) gives 
$$\mathbb E_S\kll(\data|\hat\mu_0)\le\frac{(2.4+2)\,l\Lambda_l}{n}=\frac{4.4\,l\Lambda_l}{n},$$
and (ii) gives, with probability at least $1-14\delta$ since $\delta\ge e^{-K^d}$ implies
$\log\log\frac1\delta\le\log K^d$,
\[
\kll(\data|\hat\mu_0)\le130000\,\frac{2\,l\Lambda_l+2\log(K^d)\log\frac1\delta}{n}=260000\,\frac{\Gamma_\delta}{n}.
\]
We observe that neither $\hat\mu_0^{(D_n)}$ nor the events $A\cap B\cap E$, $A\cap B$ depend on $l$, so both bounds hold
simultaneously for every $l$ with $\data\in\mathcal P_l(\mathcal X_0)$.

To finish, we transfer this bounds to the sampler. For every fixed sample $S$, $\hat\mu_0^{(D_n)}$ is a strictly positive law on
$\mathcal X_0$ and $\hat s^{(D_n)}$ is its plug-in score. Thus,  Lemma~\ref{lem:sampler_mask}, applied conditionally
on $S$, gives $$\kll(\data|\hat\nu)=\kll(\data|\hat\mu_0)$$ for the masked kernel and
$$\kll(\data|\hat\nu)\le\kll(\data|\hat\mu_0)+\frac1n$$ for the uniform kernel, with for the latter the additional hypothesis $\bar\sigma_T\ge\log(dKn)$.
Since $\frac1n\le\frac{l\Lambda_l}{n}\le\frac{\Gamma_\delta}{n}$, $\hat\nu$ satisfies the upper bounds of
Corollary~\ref{prop:minimax_mask_prob} with constants $5.4$ and $260001$, with the same event for every $l$. This concludes.
\end{proof}

Above we proved that we reach minimax rates with $\gamma=D_n.$ For completeness, we prove here also that the same rates are reachable if $\gamma$ is inside of an interval dependent of the sparsity parameter $l$. Since this parameter is unknown, the previous corollary gave us a way to implement a valid choice of $\gamma$ despite the fact that we don't know $l.$ Moreover, let us notice that $D_n$ is in general close tu $\keff,$ not $l,$ so it may be outside of the interval. This means Corollary~\ref{cor:adaptive_gamma} and Corollary~\ref{prop:minimax_mask_prob} prove the same rates for potentially different values of $\gamma.$ 
\begin{proof}[Proof of Corollary~\ref{prop:minimax_mask_prob}]

Since the proof is the same for both kernels, we are going to explain it for both masked and uniform diffusion. Throughout $\mudata\in\cP_l(\cX_0)$, so $\koeff(\mudata,m)\le\keff\le l$ for every $m\le n$. Also, since $l\le K^d/55$, it holds $\Lambda_l\ge\log(55e)>5$. Furthermore, $K^d\ge110$ implies $\Lambda_l\le1+\log K^d\le2\log K^d$; and $1\le l\le l\Lambda_l\le\Gamma_\delta$.

We fix a sample $S$ and we call the its potentially associated density $\hat\mu_0^{(\gamma(S))}$, which is a strictly positive law on $\cX_0,$ and $\hat s$ is its score. Lemma~\ref{lem:sampler_mask} gives  $$\KL(\mudata|\hat\nu)=\KL(\mudata|\hat\mu_0)$$
for the masked case, while Lemma~\ref{lem:sampler_uniform}, under the hypothesis $\bar\sigma_T\ge\log(dKn)$, gives 
$$\KL(\mudata|\hat\nu)\le\KL(\mudata|\hat\mu_0)+\frac1n$$
for the uniform kernel. As $\frac1n\le \frac{l\Lambda_l}{n}\le\frac{\Gamma_\delta}{n}$, it suffices to prove the upper bounds for $\hat\mu_0$ with constant $C-1$.

Let $\gamma_0:=\max\{1,c_0l\}$ be a deterministic element of $[1,\min\{n,K^d\}]$ since $c_0\le1$ and $l\le\min\{n,K^d\}$; by assumption $\gamma\ge\gamma_0$ surely. If $c_0l\ge1$ then $L_{\gamma_0}=\log\frac{eK^d}{c_0l}=\Lambda_l+\log\frac1{c_0}$; on the other side, if $c_0l<1$ then $\gamma_0=1$ and $L_{\gamma_0}=\Lambda_l+\log l\le\Lambda_l+\log\frac1{c_0}$. Since $u\mapsto L_u$ is nonincreasing,
\[
L_\gamma\le L_{\gamma_0}\le\Lambda_l+\log\tfrac1{c_0}\le c_1\Lambda_l,\qquad c_1:=1+\frac{\log(1/c_0)}{\log(55e)},
\]
where we used $\Lambda_l\ge\log(55e)$.

We apply Lemma~\ref{lem:upper_bound_expectation_mask}, with $n\ge110\ge8$ and this $\gamma_0$: $\Prob_S(\gamma<\gamma_0)=0$, $\E_S[\gamma]\le C_0l\Lambda_l$, $\keff\le l$, $\koeff(\mudata,n/8)\le l$. With the explicit constants of the proof of Lemma~\ref{lem:upper_bound_expectation_mask},
\[
\E_S\KL(\mudata|\hat\mu_0)\le\frac{12\,l+7C_0\,l\Lambda_l+(23+8L_{\gamma_0})\,l}{n}\le\frac{(12+7C_0+31c_1)\,l\Lambda_l}{n}.
\]
Thanks to the results above, this gives the first item with $C_1:=13+7C_0+31c_1$.

We now work on the upper bound of the second item. Let $A$ be the event of Lemma~\ref{lem:upper_bound_probability_mask}, which is applicable since $n\ge13$, $K^d\ge3$, $\delta\in(e^{-n/6},e^{-2})$, we have  $\Prob_S(A)\ge1-12\delta$, on which the bound holds simultaneously for every $\gamma\in[1,\min\{n,K^d\}]$, hence for the random $\gamma$. On $A$, with the explicit constants of the proof of Lemma~\ref{lem:upper_bound_probability_mask}, $\keff,\koeff(\mudata,\frac{n}{112})\le l\le l\Lambda_l, \,\gamma\le C_0\Gamma_\delta,\, 
L_\gamma\le c_1\Lambda_l\le2c_1\log K^d, \,\log(1/\delta)\le\log(K^d)\log(1/\delta),\, \log\log(1/\delta)\le\log K^d,$
since $\delta\ge e^{-K^d},$ it holds
\begin{align*}
    \KL(\mudata|\hat\mu_0)&\le\frac{(63100+336c_1)\,l\Lambda_l+7C_0\Gamma_\delta+(106042+13592c_1)\log(K^d)\log\frac1\delta}{n}\\
    &\le\frac{(169150+7C_0+13928c_1)\,\Gamma_\delta}{n}.
\end{align*}
With the results above, this gives the upper bound of the second item with $C_2:=169151+7C_0+13928c_1$. Set $C:=\max\{C_1,C_2\}$, which depends only on $(c_0,C_0)$.

For completeness, we also prove that the upper bounds of the first and second item match the minimax lower bounds up to universal constants.
Let $S\mapsto\nu_S\in\cP(\cX)$ be measurable. Define $\tilde\nu_S:=\nu_S(\cdot\mid\cX_0)$ if $\nu_S(\cX_0)>0$ and $\tilde\nu_S$ a fixed element of $\cP(\cX_0)$ otherwise: for the uniform kernel $\cX=\cX_0$ and $\tilde\nu_S=\nu_S$ in particular. 
Now, for every $\mu\in\cP(\cX_0)$, 
$$\KL(\mu|\nu_S)\ge\KL(\mu|\tilde\nu_S).$$ 
In fact, if $\nu_S(\cX_0)=0$ the left side is $+\infty$, otherwise $\tilde\nu_S\ge\nu_S$ on $\cX_0$. The map $S\mapsto\tilde\nu_S$ is an estimator $\cX_0^n\to\cP(\cX_0)$, so Corollary~\ref{cor:mourtada_lower_bound} applies with $|\cX|=K^d$; its hypotheses $n,K^d\ge110$, $l\in\{2,\dots,\min(n,\lfloor K^d/55\rfloor)\}$ are the same as us, and $\Gamma_\delta$ is exactly the numerator in \eqref{eq:minimax-sparse-lower}.

In the first item case, let now take $\delta_0:=e^{-3}\in(e^{-n},e^{-2})$, together with $n\ge110$. Corollary~\ref{cor:mourtada_lower_bound} gives $\mudata\in\cP_l(\cX_0)$ with $$\Prob_S\big(\KL(\mudata|\tilde\nu_S)\ge\frac{l\Lambda_l+3\log K^d}{320n}\big)\ge e^{-3},$$ hence, since $Z\ge a\mathbf 1\{Z\ge a\}$ for $Z\ge0$, we have
\[
\E_S\KL(\mudata|\nu_S)\ge\E_S\KL(\mudata|\tilde\nu_S)\ge\frac{e^{-3}}{320}\,\frac{l\Lambda_l}{n}.
\]
On the other hand, for the second item, for $\delta\in(e^{-\min\{n/6,K^d\}},e^{-2})\subseteq(e^{-n},e^{-2})$, Corollary~\ref{cor:mourtada_lower_bound} gives $\mudata\in\cP_l(\cX_0)$ with 
$$\Prob_S\big(\KL(\mudata|\tilde\nu_S)\ge\frac{\Gamma_\delta}{320n}\big)\ge\delta,$$
and $\KL(\mudata|\nu_S)\ge\KL(\mudata|\tilde\nu_S)$.

Both lower bounds hold with the universal constant $c:=\frac{1}{320e^3}$.
\end{proof}

\section{Proofs of Section~\ref{sec:statistical_rates:general}}\label{appendix:sec_entropy}

We first state precisely the mixed loss of Section~\ref{sec:entropy} for a general kernel. Let $\mu^U_t:=P_{0,t}\Unif(\cX_0)$ be the forward marginal starting from uniform data, let $s^\star_U$ be its score, defined as in \eqref{definition_true_score} and $\mathcal L^U_t$ the score-entropy loss. For $\gamma>0$, the mixed loss interpolates between the empirical loss \eqref{eq:loss_empirical} and $\mathcal L^U_t$:
\begin{equation}\label{eq:mixed_loss}
\hat{\mathcal L}^M_t(s):=\frac{n}{n+\gamma}\,\hat{\mathcal L}_t(s)+\frac{\gamma}{n+\gamma}\,\mathcal L^U_t(s).
\end{equation}
Since the score-entropy loss is linear in the clean distribution and $\frac{n}{n+\gamma}\mu^{(n)}_0+\frac{\gamma}{n+\gamma}\Unif(\cX_0)=\hat\mu_0^{(\gamma)}$, the mixed loss is, up to a constant, the score-entropy loss of $\hat\mu_0^{(\gamma)}$. The following lemma makes this precise.

\begin{lemma}
Let $\gamma>0$, $\hat\mu_t:=P_{0,t}\muhat$ and $\hat s(x,y,t):=\frac{\hat\mu_t(x)}{\hat\mu_t(y)}$ as in \eqref{eq:hat_s_uniform}. For every $t\in(0,T]$ and every score $s$ positive where $Q_t(y,x)>0$,
\[
\hat{\mathcal L}^M_t(s)=\sum_{y\in\cX}\hat\mu_t(y)\sum_{x\ne y}Q_t(y,x)\,D\big(\hat s(x,y,t),s(x,y,t)\big)+c_t,
\]
where $c_t\in\mathbb R$ does not depend on $s$. In particular, $\hat s$ minimizes $s\mapsto\int_0^T\hat{\mathcal L}^M_t(s)\dd t$, and a score $s$ is another minimizer if and only if $s(x,y,t)=\hat s(x,y,t)$ for almost every $t\in(0,T]$ and every $x\neq y$ with $Q_t(y,x)>0$.
\end{lemma}
\begin{proof}
Fix $t\in(0,T]$ and a score $s$ positive where $Q_t(y,x)>0$. Since $Q_t(y,x)>0$ and $\rho_t(y)=0$ it must hold that $\rho_t(x)=0$ also. Moreover, since every state of $\cX$ can be reached from $\cX_0$, it follows that $\muhat>0$ on $\cX_0$ gives $\hat\mu_t>0$ on $\cX$.

By \eqref{eq:loss_empirical}, averaging over the sample with $\mu^{(n)}_t=\frac1n\sum_kP_{0,t}(\cdot,X^{(k)}_0)$,
\begin{equation}\label{eq:mixed_emp}
\hat{\mathcal L}_t(s)=\sum_{y}\sum_{x\neq y}Q_t(y,x)\big(\mu^{(n)}_t(y)\,s-\mu^{(n)}_t(x)\log s\big).
\end{equation}
Similarly, we have $\mu^U_t(y)s^\star_U(x,y,t)=\mu^U_t(x)$ whenever $Q_t(y,x)>0$, so by \eqref{definition_score_entropy_loss},
\begin{gather}\label{eq:mixed_unif}
\mathcal L^U_t(s)=\sum_{y}\sum_{x\neq y}Q_t(y,x)\big(\mu^U_t(y)\,s-\mu^U_t(x)\log s\big)+c^U_t,\\ c^U_t:=\sum_y\mu^U_t(y)\sum_{x\neq y}Q_t(y,x)\,\phi\big(s^\star_U(x,y,t)\big).\nonumber
\end{gather}
Both expressions are linear in the measure, and $\frac{n}{n+\gamma}\mu^{(n)}_t+\frac{\gamma}{n+\gamma}\mu^U_t=P_{0,t}\muhat=\hat\mu_t$ by linearity of $P_{0,t}$. Hence, by \eqref{eq:mixed_loss},
\begin{align*}
    \hat{\mathcal L}^M_t(s)&=\sum_{y}\sum_{x\neq y}Q_t(y,x)\big(\hat\mu_t(y)\,s-\hat\mu_t(x)\log s\big)+\tfrac{\gamma}{n+\gamma}\,c^U_t\\
    &=\sum_y\hat\mu_t(y)\sum_{x\neq y}Q_t(y,x)\big(s-\hat s\log s\big)+\tfrac{\gamma}{n+\gamma}\,c^U_t,
\end{align*}
where we used $\hat\mu_t(x)=\hat\mu_t(y)\,\hat s$. Since $s-\hat s\log s=D(\hat s,s)-\phi(\hat s)$ by definition of $D$, this is the claimed identity with
\[
c_t:=\tfrac{\gamma}{n+\gamma}\,c^U_t-\sum_y\hat\mu_t(y)\sum_{x\neq y}Q_t(y,x)\,\phi(\hat s).
\]
\end{proof}

\begin{proof}[Proof of Lemma~\ref{lem:sampler_uniform}]
To begin this proof, we recall that the matrix $Q_0$ is in this case symmetric, so in particular $P_{0,t}$ is doubly stochastic, in the sense that both columns and rows sum to 1. Furthermore, since the change of one token $b_0$ to $b$ has probability $\frac{1-\alpha_t}{K}+\alpha_t\mathbf{1}_{b=b_0}\leq \frac1K+\alpha_t,$ it follows that 
$$P_{0,t}(x, x_0)\leq \frac{1}{K^d}(1+K\alpha_t)^d$$
for all $x,x_0\in \mathcal{X}_0.$
Using now this inequality, we have that $$\mu_T(x)=\sum_{x_0} P_{0,T}(x, x_0)\data(x_0)\leq \frac{1}{K^d}(1+K\alpha_T)^d,$$ so
\begin{equation}
    \kll(\mu_T\mid \Unif(\mathcal{X}_0))=\sum_x \mu_T(x)\log(K^d\mu_T(x))\leq d \log(1+K\alpha_T)\leq dK\alpha_T\leq \frac1n.
\end{equation}
Notice that above we used that $\alpha_T=e^{-\bar\sigma_T}\leq (dKn)^{-1}.$

By Lemma~\ref{lem:A}, $M_\gamma^{-1}\le\hat s\le M_\gamma$ on $\mathcal X_0^2\times[0,T]$ with $M=\frac{K^d(n+\gamma)}{\gamma}$, and $\hat s$ is continuous in $t$. Hence $|\hat Q_u(x,y)|\le M_\gamma\sigma_{T-u}\max|Q_0|$, which is integrable on $[0,T]$ because we assumed $\bar\sigma_T<\infty$. So, we have a unique absolutely continuous solution $(\nu_u)_{u\in[0,T]}$ on the closed interval, $\hat\nu=\nu_T$, and $\hat\nu_{T-T_0}\to\nu_T$ as $T_0\downarrow0$ by continuity. For every $T_0\in(0,T)$ the hypotheses of Lemma~\ref{lem:transfer} hold with $\mu_\infty=\mathrm{Unif}(\mathcal X_0)>0$ and $s=\hat s$, and Lemma~\ref{lem:B} gives 
$$\int_{T_0}^T\mathcal{L}_t(\hat s)\,\mathrm dt\le\mathrm{KL}(\mu_{\mathrm{data}}|\hat\mu_0^{(\gamma)}).$$
Together with the bound $\mathrm{KL}(\mu_T|\mathrm{Unif}(\mathcal X_0))\le\frac1n$ above, Lemma~\ref{lem:transfer} yields
\begin{equation}\label{eq:T0bound}
\mathrm{KL}(\mu_{T_0}\,|\,\nu_{T-T_0})\le\mathrm{KL}(\mu_{\mathrm{data}}\,|\,\hat\mu_0^{(\gamma)})+\frac1n .
\end{equation}
Finally, $\mu_{T_0}=e^{\bar\sigma_{T_0}Q_0}\mu_{\mathrm{data}}\to\mu_{\mathrm{data}}$ as $T_0\downarrow0$ since $\bar\sigma_{T_0}\to0$, and on a finite set $(\mu,\nu)\mapsto\mathrm{KL}(\mu|\nu)$ is jointly lower semicontinuous, since it is a supremum of continuous functions by the Donsker--Varadhan formula $\mathrm{KL}(\mu|\nu)=\sup_f\{\sum_xf(x)\mu(x)-\log\sum_xe^{f(x)}\nu(x)\}$. Hence
\[
\mathrm{KL}(\mu_{\mathrm{data}}\,|\,\hat\nu)\le\liminf_{T_0\downarrow0}\mathrm{KL}(\mu_{T_0}\,|\,\nu_{T-T_0})\le\mathrm{KL}(\mu_{\mathrm{data}}\,|\,\hat\mu_0^{(\gamma)})+\frac1n .
\]

\end{proof}

\begin{proof}[Proof of Theorem~\ref{thm:uniform}]
We start fixing a sample $S$; $\gamma(S)\in[1,\min\{n,K^d\}]$, $\hat\mu_0^{(\gamma(S))}>0$ on $\cX_0$ and $\hat s$ its score.
Thanks to Lemma~\ref{lem:sampler_uniform}, we know that 
$$\KL(\mudata|\hat\nu)\le\KL(\mudata|\hat\mu_0^{(\gamma(S))})+\frac1n.$$Taking expectations and using Lemma~\ref{lem:upper_bound_expectation_mask} gives the first item. Notice that $C$ of Theorem~\ref{thm:upper_bounds_mask} is replaced by $C+1$ here, because $\frac1n\le\frac{\keff}{n}$, thanks to $\keff\ge\keff(\mudata,1)=1$ for $n\ge1$. Intersecting with the event $A$ of Lemma~\ref{lem:upper_bound_probability_mask}, which does not depend on $\gamma$, gives the second item in the same way. Corollaries~\ref{cor:sparse_mask},\ref{prop:minimax_mask_prob},\ref{cor:adaptive_gamma} then transfer perfectly, as recorded in their proofs.
\end{proof}

%% file: sections/appendix_experiments.tex
\section{Experiment details}\label{app:experiments}

All experiments use the masked kernel with $\alpha_t=e^{-t}$ and $\kappa_t=\alpha_t/(1-\alpha_t)$. Seeds are fixed throughout.

\paragraph{Targets for label smoothing.} We take $K=4$ and $d=8$, so $K^d=65{,}536$. A $0/1$ matrix $A\in\{0,1\}^{K\times K}$ specifies the allowed token transitions. The first token is drawn from the stationary law of the chain, and each later token from the row-normalized weights, which are $\mathrm{Dirichlet}(2)$ on the allowed entries. The support is therefore exactly the set of length-$d$ paths through $A$, and $l=\mathbf 1^\top A^{d-1}\mathbf 1$. We enumerate all $2^{16}$ matrices and keep the $25{,}696$ that are strongly connected and give every state at least one successor. For each target $l\in\{16,64,256,1024,4096\}$ and each seed, we draw $A$ uniformly among those with $l$ within $\pm5\%$ of the target. We use $n\in\{64,256,1024\}$ and 5 seeds, each seed fixes $A$, the weights and the sample, and the same sample is used for every $\gamma$.

\paragraph{Label-smoothed estimator.} By Lemma~\ref{lem:sampler_mask}, the masked sampler outputs exactly $\hat\mu_0^{(\gamma)}$. We therefore compute $\kll(\data\mid\hat\mu_0^{(\gamma)})$ exactly by summing over $\mathcal V^d$, with no simulation.

\paragraph{Clipped empirical score.} For the masked kernel, the estimator of \citet[Eq.~(11)]{zhang2026} reads $\hat s_t(x,y)=\mathrm{clip}(\kappa_t\psi_c^{(n)}(a),\varepsilon,\kappa_t)$, and we set $\hat s_t=\varepsilon$ on unseen contexts ($n_c=0$). Following their Theorem~3.8, we use early stopping at $t_0=1/n$, terminal time $T=\log n$, and $\varepsilon=c_0\log n/n$, with $c_0\in\{0.1,1,10\}$ chosen per $(l,n)$ using the true KL.

Since the floor $\varepsilon/\kappa_t$ depends on $t$, this score is improper, and the random-order sampler is not exact for it. We instead compute its output law by integrating the reverse Kolmogorov equation over all $(K+1)^d=390{,}625$ partially masked states. We use RK4 on 400 log-spaced times from $T$ to $t_0$. Positions still masked at $t_0$ are unmasked one at a time in uniformly random order, with each token drawn with probability $\propto\hat s_{t_0}(\cdot,y)$. As a check, running the same solver with the proper score $\kappa_t\hat\psi_c$ reproduces $\hat\mu_0^{(\gamma)}$ with a maximum relative error of $7.2\times10^{-6}$.

\paragraph{Training.} We use the sparse Markov chain of \citet[App.~E]{zhang2026}: $K=100$, $d=12$, $20$ significant successors per state, $\epsilon_{\mathrm{bg}}=0.01$ and $p_{\min}=10^{-3}$. Held-out sequences are drawn fresh from the same chain. All training sequences are distinct, so $D_n=n$.

The model is their transformer (4 layers, width 128, 4 heads, $\approx0.83$M parameters). We train it with AdamW (learning rate $5\times10^{-4}$, $\beta=(0.9,0.999)$, weight decay $10^{-3}$, batch size 128). To train on $\hat\mu_0^{(\gamma)}$, each sequence is replaced by a uniform element of $\mathcal V^d$ with probability $p=\gamma/(n+\gamma)$, for $\gamma\in\{0,64,256,1024,4096\}$ and $n\in\{1024,2896,8192\}$. Each arm runs for $\lceil 6000/(1-p)\rceil$ iterations, so that all arms see the same expected number of real sequences. We evaluate the held-out NELBO at 40 log-spaced iterations starting at 10. Each of the 5 seeds fixes both the sample and the initialization.